\documentclass{article}

\usepackage[nonatbib,final]{neurips_2023}

\usepackage[utf8]{inputenc} %
\usepackage[T1]{fontenc}    %
\usepackage{url}            %
\usepackage{booktabs}       %
\usepackage{amsfonts}       %
\usepackage{nicefrac}       %
\usepackage{microtype}      %
\usepackage{ying}
\usepackage{wrapfig}
\usepackage{multirow}
\usepackage{xspace}
\usepackage{mathrsfs}
\usepackage{times}
\usepackage[ruled,vlined]{algorithm2e}
\usepackage{amsmath}
\allowdisplaybreaks
\usepackage[table]{xcolor}
\usepackage{color,colortbl}
\definecolor{LightCyan}{rgb}{0.88,1,1}
\definecolor{GrayDark}{RGB}{191,191,191}
\definecolor{GrayLight}{RGB}{217,217,217}
\usepackage[colorlinks=true,linkcolor=blue, urlcolor=magenta,citecolor=blue, anchorcolor=blue]{hyperref}
\usepackage{wrapfig}
\usepackage{amsthm}
\usepackage[font=small,labelfont=bf]{caption}
\usepackage{adjustbox}
\usepackage{subcaption}
\usepackage{graphicx}
\newtheorem{theorem}{Theorem}

\newcommand{\ie}{\textrm{i.e.}}
\newcommand{\eg}{\textrm{e.g.}}

\newcommand{\cut}[1]{}
\newcommand{\xhdr}[1]{\noindent{{\bf #1.}}}

\def \calib {\textrm{calib}}
\def \train {\textrm{train}} 
\def \test {\textrm{test}}
\newcommand{\greencheck}{{\color{darkpastelgreen}\CheckmarkBold}}
\newcommand{\redmark}{{\color{red}\ding{55}}}
\newcommand{\std}[1]{{\scriptsize{$\pm$#1}}}

\newtheorem{lemma}[theorem]{Lemma}
\newtheorem{assumption}[theorem]{Assumption}

\definecolor{ForestGreen}{RGB}{34,139,34}
\newcommand{\revise}[1]{{\color{black} #1}}

\usepackage{bbding}
\usepackage{pifont}
\definecolor{darkpastelgreen}{rgb}{0.01, 0.75, 0.24}

\SetCommentSty{mycommfont}

\newcommand{\mname}{\textsc{CF-GNN}\xspace}

\title{Uncertainty Quantification over Graph with \\ Conformalized Graph Neural Networks}

\author{%
  Kexin Huang$^{1}$ \quad Ying Jin$^{2}$ \quad Emmanuel Cand{\`e}s$^{2,3}$ \quad Jure Leskovec$^{1}$ \\
  \\
  $^1$ Department of Computer Science, Stanford University \\
  $^2$ Department of Statistics, Stanford University \\
  $^3$ Department of Mathematics, Stanford University\\
 \texttt{kexinh@cs.stanford.edu}, \texttt{ying531@stanford.edu},\\
 \texttt{candes@stanford.edu}, \texttt{jure@cs.stanford.edu}
}

\begin{document}

\maketitle

\begin{abstract}
  Graph Neural Networks (GNNs) are powerful machine learning prediction models on graph-structured data. However, GNNs lack rigorous uncertainty estimates, limiting their reliable deployment in \revise{settings where the cost of errors is significant}. We propose conformalized GNN (\mname), extending conformal prediction (CP) to graph-based models for guaranteed uncertainty estimates. Given an entity in the graph, \mname produces a prediction set/interval that provably contains the true label with pre-defined coverage probability (\eg~90\%). We establish a permutation invariance condition that enables the validity of CP on graph data and provide an exact characterization of the test-time coverage. Besides valid coverage, it is crucial to reduce the prediction set size/interval length for practical use. We observe a key connection between non-conformity scores and network structures, which motivates us to develop a topology-aware output correction model that learns to update the prediction and produces more efficient prediction sets/intervals. Extensive experiments show that \mname achieves any pre-defined target marginal coverage while significantly reducing the prediction set/interval size by up to 74\% over the baselines. It also empirically achieves satisfactory conditional coverage over various raw and network features. 
\end{abstract}
\section{Introduction}

Graph Neural Networks (GNNs) have shown great potential in learning representations for graph-structured data, which has led to their widespread adoption in weather forecasting~\cite{lam2022graphcast}, drug discovery~\cite{li2022graph}, and recommender systems~\cite{wu2022graph}, etc. As GNNs are increasingly deployed in  high-stakes settings, it is important to understand the uncertainty in the predictions they produce. One prominent approach to uncertainty quantification is to construct a prediction set/interval that informs a plausible range of values the true outcome  may take. A large number of methods have been proposed to achieve this goal~\cite{hsu2022makes,zhang2020mix,lakshminarayanan2017simple,wang2021confident}. However, these methods often lack theoretical and empirical guarantees regarding their validity, \ie~the probability that the prediction set/interval covers the outcome~\cite{angelopoulos2020uncertainty}. This lack of rigor hinders their reliable deployment in situations where errors can be consequential.

Conformal prediction~\cite{vovk2005algorithmic} (CP) is a framework for producing statistically guaranteed uncertainty estimates. Given a user-specified miscoverage level $\alpha\in(0,1)$, it
leverages a set of ``calibration'' data to
output prediction sets/intervals for new test points that provably include the true outcome with probability at least $1-\alpha$. Put another way, the conformal prediction sets provably only miss the test outcomes  at most $\alpha$ fraction of the time. With its simple formulation, clear guarantee and distribution-free nature, it has been successfully applied to various problems in computer vision~\cite{angelopoulos2020uncertainty,bates2021distribution}, causal inference~\cite{lei2021conformal,jin2023sensitivity,yin2022conformal},  time series forecasting~\cite{gibbs2021adaptive,zaffran2022adaptive}, and drug discovery~\cite{jin2022selection}.

Despite its success in numerous domains, conformal prediction has remained largely unexplored in the context of graph-structured data. One primary challenge  is that it is unclear if 
the only, yet crucial, assumption for CP---exchangeability between the test and calibration samples---holds for graph data. When applying conformal prediction, exchangeability is usually ensured by independence among the trained model, the calibration data, and test samples (see Appendix~\ref{appendix:full_vs_split} for more discussion). However, in the transductive setting, GNN training employs all nodes within the same graph--including test points--for message passing, creating intricate dependencies among them. Thus, to deploy conformal prediction for graph data, the first challenge is to identify situations where valid conformal prediction is possible given a fitted GNN model that already involves test information.

Efficiency is another crucial aspect of conformal prediction for practical use: a prediction set/interval with an enormous set size/interval length might not be practically desirable even though it achieves valid coverage. %
Therefore, the second major challenge is to develop a graph-specific approach to reduce the size of the prediction set or the length of the prediction interval (dubbed as \emph{inefficiency} hereafter for brevity) while retaining the attractive coverage property of conformal prediction.

\xhdr{Present work} We propose conformalized GNN (\mname),\footnote{See Figure~\ref{fig:method} for an overview. The code is available at \url{https://github.com/snap-stanford/conformalized-gnn}.}
extending conformal prediction to GNN for rigorous uncertainty quantification over graphs. We begin by establishing the validity of conformal prediction for graphs. We show that in the transductive setting, regardless of the dependence among calibration, test, and training nodes, standard conformal prediction~\cite{vovk2005algorithmic} is valid as long as the score function (whose definition will be made clear in Section~\ref{sec:background}) is invariant to the ordering of calibration and test samples. This condition is easily satisfied by popular GNN models. Furthermore, we provide an exact characterization of the empirical test-time coverage. 

Subsequently, we present a new approach that learns to optimize the inefficiencies of conformal prediction. 
We conduct an empirical analysis which reveals that inefficiencies are highly correlated along the network edges. 
Based on this observation, we add a topology-aware correction model that updates the node predictions based on their neighbors. This model is trained by minimizing a differentiable efficiency loss that simulates the CP set sizes/interval lengths. In this way, unlike the raw prediction that is often optimized for prediction accuracy, the corrected GNN prediction is optimized to yield smaller/shorter conformal prediction sets/intervals. Crucially, our approach aligns with the developed theory of graph exchangeability, ensuring valid coverage guarantees while simultaneously enhancing efficiency.

We conduct extensive experiments across 15 diverse datasets for both node classification and regression with 8 uncertainty quantification (UQ) baselines, covering a wide range of application domains. While all previous UQ methods fail to reach pre-defined target coverage, \mname achieves the pre-defined empirical marginal coverage. It also significantly reduces the prediction set sizes/interval lengths by up to 74\% compared with a direct application of conformal prediction to GNN. Such improvement in efficiency does not appear to sacrifice adaptivity: we show that \mname achieves strong empirical conditional coverage over various network features.

\begin{figure}[t]
    \centering
    \includegraphics[width=1\textwidth]{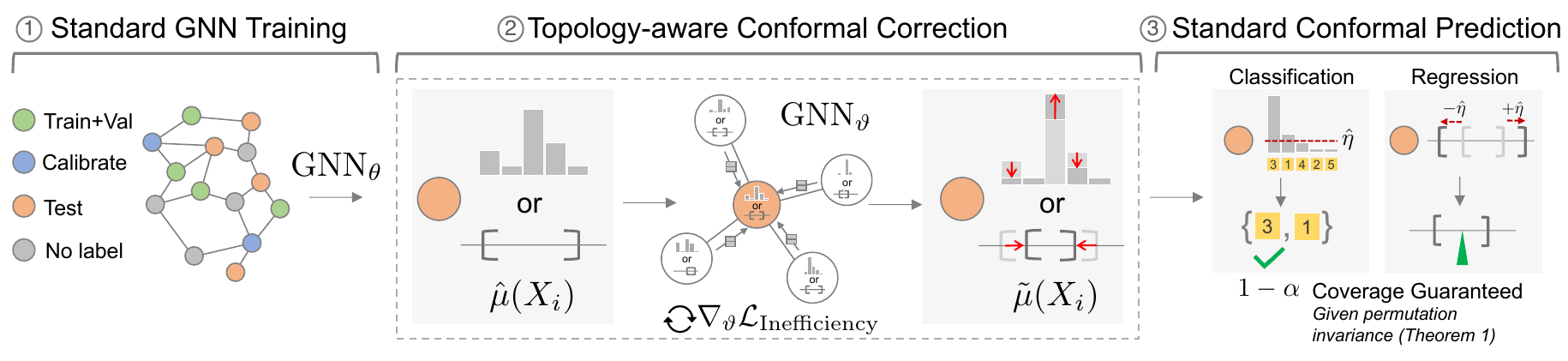}
    \caption{Conformal prediction for graph-structured data. (1) \underline{GNN training}. We first use standard GNN training to obtain a base GNN model ($\mathrm{GNN}_\theta$) that produces prediction scores $\hat{\mu}(X_i)$ for node $i$. It is fixed once trained. (2) \underline{Conformal correction}. Since the training step is not aware of the conformal calibration step, \revise{the size/length of prediction sets/intervals (\ie~efficiency)} are not optimized. We propose a novel correction step that learns to correct the prediction to achieve desirable properties such as efficiency. We use a topology-aware correction model $\mathrm{GNN}_\vartheta$ that takes $\hat{\mu}(X_i)$ as the input node feature and aggregates information from its local subgraph to produce an updated prediction $\tilde{\mu}(X_i)$. $\vartheta$ is trained by simulating the conformal prediction step and optimizing a differentiable inefficiency loss. (3) \underline{Conformal prediction}. We prove that in a transductive random split setting, graph exchangeability holds (Section~\ref{sec:theory}) given permutation invariance. Thus, standard CP can be used to produce a prediction set/interval based on $\tilde\mu$ that includes true label with pre-specified coverage rate 1-$\alpha$.}
    \label{fig:method}
\end{figure}

\section{Background and Problem Formulation}\label{sec:background}

Let $G = (\mathcal{V}, \mathcal{E}, \mathbf{X})$ be a graph, where $\mathcal{V}$ is a set of nodes, $\mathcal{E}$ is a set of edges, and $\mathbf{X} = \{\mathbf{x}_v\}_{v\in \cV}$ is the attributes, where $\mathbf{x}_v \in \mathbb{R}^d$ is a $d$-dimensional feature vector for node $v \in \mathcal{V}$. The label of node $v$ is $y_v\in \cY$. For classification, $\mathcal{Y}$ is the discrete set of possible label classes. For regression, $\cY=\mathbb{R}$. 

\xhdr{Transductive setting} 
We focus on transductive node classification/regression problems with random data split. In this setting, the graph $G$ is fixed. 
At the beginning, we have access to $\{(\mathbf{x}_v,y_v)\}_{v\in \cD}$ as the ``training'' data, 
as well as test data $\cD_\test$ with unknown labels $\{y_v\}_{v\in \cD_\test}$. Here $\cD$ and $\cD_\test$ are disjoint subsets of $\cV$. We work on the prevalent random split setting where nodes in $\cD$ and $\cD_\test$ are randomly allocated from the entire graph, 
and the test sample size is $m = |\cD_\test|$. 
The training node set $\cD$ is then randomly split into $\mathcal{D}_\mathrm{train}/\mathcal{D}_\mathrm{valid}/\mathcal{D}_\mathrm{calib}$ of fixed sizes, the training/validation/calibration set, correspondingly. 
A perhaps nonstandard point here is that we withhold a subset $\cD_\calib$ as ``calibration'' data in order to apply conformal prediction later on. 
During the training step, the data
$\{(\mathbf{x}_v,y_v)\}_{v\in \cD_{\train}\cup \cD_\textrm{valid}}$, the attribute information in $\{\mathbf{x}_v\}_{v\in \mathcal{D}_\mathrm{calib}\cup\mathcal{D}_\mathrm{test}}$ and the entire graph structure $(\cV,\cE)$ are available to the GNN to compute training nodes representations, while  $\{y_v\}_{v\in \mathcal{D}_\mathrm{calib}\cup\mathcal{D}_\mathrm{test}}$ are not seen.

\xhdr{Graph Neural Networks (GNNs)} GNNs learn compact representations that capture network structure and node features. A GNN generates outputs through a series of propagation layers~\cite{gilmer2017neural}, where propagation at layer $l$ consists of the following three steps: (1) \underline{Neural message passing}. GNN computes a message $\mathbf{m}^{(l)}_{uv} = \textsc{Msg}(\mathbf{h}_u^{(l-1)}, \mathbf{h}_v^{(l-1)})$ for every linked nodes $u,v$ based on their embeddings from the previous layer $\mathbf{h}_u^{(l-1)}$ and $\mathbf{h}_v^{(l-1)}$. (2)~\underline{Neighborhood aggregation}. The messages between node $u$ and its neighbors $\mathcal{N}_u$ are aggregated as $\hat{\mathbf{m}}^{(l)}_{u} = \textsc{Agg}({\mathbf{m}^{(l)}_{uv} | v \in \mathcal{N}_u})$. (3) \underline{Update}. Finally, GNN uses a non-linear function to update node embeddings as $\mathbf{h}^{(l)}_u = \textsc{Upd}(\hat{\mathbf{m}}^{(l)}_{u}, \mathbf{h}^{(l-1)}_u)$ using the aggregated message and the embedding from the previous layer. The obtained final node representation is then fed to a classifier or regressor to obtain a prediction $\hat\mu(X)$. 

\xhdr{Conformal prediction} 
In this work, we focus on the computationally efficient split conformal prediction method~\cite{vovk2005algorithmic}.\footnote{See Appendix~\ref{appendix:full_vs_split} for discussion on full and split conformal prediction. We refer to the split conformal prediction when we describe conformal prediction throughout the paper.} Given a predefined miscoverage rate $\alpha \in [0,1]$, 
it proceeds in three steps: (1) \underline{non-conformity scores}. CP first obtains any heuristic notion of uncertainty called non-conformity score $V\colon \cX \times \cY \to \mathbb{R}$. Intuitively, $V(x, y)$ measures how $y$ "conforms" to the prediction at $x$. An example is the predicted probability of a class $y$ in classification or the residual value $V(x,y)=|y-\hat\mu(x)|$ in regression for a predictor $\hat\mu\colon \cX\to \cY$. (2) \underline{Quantile computation}. CP then takes the $1-\alpha$ quantile of the non-conformity scores computed on the calibration set. 
Let $\{(X_i,Y_i)\}_{i=1}^n$ be the calibration data, where $n=|\cD_\calib|$, and compute $\hat{\eta} = \mathrm{quantile}(\{V(X_1, Y_1), \cdots, V(X_n, Y_n)\}, (1-\alpha)(1+\frac{1}{n}))$. (3) \underline{Prediction set/interval construction}. Given a new test point $X_{n+1}$, CP constructs a prediction set/interval $C(X_{n+1}) = \{y\in\cY: V(X_{n+1}, y) \le \hat{\eta}\}$. If  $\{Z_i\}_{i=1}^{n+1}:=\{(X_i,Y_i)\}_{i=1}^{n+1}$ are exchangeable,\footnote{Exchangeability definition: for any $z_1,\dots,z_{n+1}$ and any permutation $\pi$ of $\{1,\dots,n+1\}$, it holds that  $\PP((Z_{\pi(1)},\dots,Z_{\pi(n+1)})=(z_{1},\dots,z_{n+1}))=\PP((Z_1,\dots,Z_{n+1})=(z_{1},\dots,z_{n+1}))$.} then $V_{n+1}:=V(X_{n+1},Y_{n+1})$ is exchangeable with $\{V_i\}_{i=1}^n$ since $\hat\mu$ is given. Thus, $\hat{C}(X_{n+1})$ contains the true label with predefined coverage rate~\cite{vovk2005algorithmic}: $P\{Y_{n+1} \in C(X_{n+1})\} = \PP\{ V_{n+1}\geq \textrm{Quantile}(\{V_1,\dots,V_{n+1}\},1-\alpha) \ge 1-\alpha$ due to exchangeability of $\{V_i\}_{i=1}^{n+1}$. %
This framework works for any non-conformity score.   \mname is similarly non-conformity score-agnostic. 
However, for demonstration, we focus on two popular  scores, described in detail below.

\xhdr{Adaptive Prediction Set (APS)} For the classification task, we use the non-conformity score in APS proposed by~\cite{romano2020classification}. It takes the cumulative sum of ordered class probabilities till the true class. Formally, given any estimator $\hat\mu_j(x)$ for the conditional probability of $Y$ being class $j$ at $X=x$, $j=1,\dots,|\cY|$, we denote the cumulative probability till the $k$-th most promising class as $V(x, k) = \sum_{j=1}^{k} \hat\mu_{\pi_{(j)}}(x)$, where $\pi$ is a permutation of $\cY$ so that $\hat\mu_{\pi(1)}(x) \geq \hat\mu_{\pi(2)}(x)\geq\cdots \geq \hat\mu_{\pi(|\cY|)}(x)$. Then, the prediction set is constructed as $C(x) = \{\pi(1), \cdots, \pi(k^*)\}$, where $k^* = \mathrm{inf}\{k: \sum_{j=1}^{k} \hat\mu_{\pi(j)}(x) \ge \hat{\eta}\}$.

\xhdr{Conformalized Quantile Regression (CQR)} For the regression task, we use CQR in~\cite{romano2019conformalized}. CQR is based on quantile regression (QR). QR \revise{obtains heuristic estimates $\hat\mu_{\alpha/2}(x)$ and $\hat\mu_{1-\alpha/2}(x)$ for the $\alpha/2$-th and $1-\alpha/2$-th conditional quantile functions of $Y$ given $X=x$.} The non-conformity score is $V(x,y) = \mathrm{max}\{\hat\mu_{\alpha/2}(x) - y, y - \hat\mu_{1-\alpha/2}(x)\}$, interpreted as the residual of true label projected to the closest quantile. The prediction interval is then $C(x) = [\hat\mu_{\alpha/2}(x) - \hat{\eta}, \hat\mu_{1-\alpha/2}(x) + \hat{\eta}]$.

In its vanilla form, the non-conformity score (including APS and CQR) in CP does not depend on the calibration and test data. That means, $\{\mathbf{x}_v\}_{v\in \cD_{\calib}\cup\cD_\test}$ are not revealed in the training process of $V$, which is the key to exchangeability. In contrast, GNN training typically leverages the entire graph, and hence the learned model depends on the calibration and test attributes in a complicated way. In the following, for clarity, we denote any non-conformity score built on a GNN-trained model as 
\$
V(x,y; \{z_v\}_{v\in\cD_{\train}\cup\cD_{\mathrm{valid}}}, \{\mathbf{x}_v\}_{v\in \cD_\calib\cup\cD_\test} , \cV, \cE)
\$
to emphasize its dependence on the entire graph, where $z_v=(\mathbf{x}_v,Y_v)$ for $v\in \cV$.

\xhdr{Evaluation metrics} The goal is to ensure valid marginal coverage while decreasing the inefficiency as much as possible. Given the test set $\mathcal{D}_\mathrm{test}$, the empirical marginal coverage is defined as $\mathrm{Coverage}:= \frac{1}{|\mathcal{D}_\mathrm{test}|} \sum_{i\in\mathcal{D}_\mathrm{test}} \mathbb{I}(Y_i \in C(X_i))$. For the regression task, inefficiency is measured as the interval length while for the classification task, the inefficiency is the size of the prediction set: $\mathrm{Ineff}:= \frac{1}{|\mathcal{D}_\mathrm{test}|} \sum_{i\in\mathcal{D}_\mathrm{test}} |C(X_i)|$. The larger the length/size, the more inefficient. Note that inefficiency of conformal prediction is different from accuracy of the original predictions. Our method does not change the trained prediction but modifies the prediction sets from conformal prediction. %

\section{Exchangeability and Validity of Conformal Prediction on Graph}\label{sec:theory}

To deploy CP for graph-structured data, we first study the exchangeability of node information under the transductive setting. We show that under a general permutation invariant condition (Assumption~\ref{assump:permute}), exchangeability of the non-conformity scores is still valid
even though GNN training uses the calibration and test information; this paves the way for applying conformal prediction to GNN models. We develop an exact characterization of the test-time coverage of conformal prediction in such settings. 
Proofs of these results are in Appendix~\ref{app:proof}.

\begin{assumption}\label{assump:permute}
For any permutation $\pi$ of $\cD_\calib\cup\cD_\test$, 
the non-conformity score $V$ obeys 
\$
& V(x,y; \{z_v\}_{v\in\cD_{\train}\cup\cD_{\text{valid}}}, \{\mathbf{x}_v\}_{v\in \cD_\calib\cup\cD_\test} , \cV, \cE) \\ 
&= V(x,y; \{z_v\}_{v\in\cD_{\train}\cup\cD_{\text{valid}}}, \{\mathbf{x}_{\pi(v)}\}_{v\in \cD_\calib\cup\cD_\test} , \cV_\pi, \cE_\pi),
\$
where $(\cV_\pi,\cE_\pi)$ represents a graph where $\cD_\calib\cup \cD_\test$ nodes (indices) are permuted according to $\pi$.
\end{assumption}

Assumption~\ref{assump:permute} imposes a permutation invariance condition for the GNN training, i.e., model output/non-conformity score is invariant to 
permuting the ordering of the calibration and test nodes (with their edges permuted accordingly) on the graph. To put it differently, different selections of calibration sets do not modify the non-conformity scores for any node in the graph.
GNN models (including those evaluated in our experiments) typically obey Assumption~\ref{assump:permute}, because they only use the structures and attributes in the graph without information on the ordering of the nodes~\cite{kipf2016semi,hamilton2017inductive,gilmer2017neural}.  
 
For clarity, we write the calibration data as $\{(X_i,Y_i)\}_{i=1}^n$, where $X_i = X_{v_i}$, and $v_i\in \cD_\calib$ 
is the $i$-th node in the calibration data under some pre-defined ordering. Similarly, the test data are 
$\{(X_{n+j},Y_{n+j})\}_{j=1}^m$, where 
$X_{n+j}=X_{v_j}$, and $v_j\in \cD_\test$ 
is the $j$-th node in the test data. 
We write 
\$
V_i  = V(X_i,Y_i; \{z_i\}_{i\in\cD_{\train}\cup\cD_{\textrm{valid}}}, \{\mathbf{x}_v\}_{v\in \cD_\calib\cup\cD_\test}, \cV, \cE),\quad i=1,\dots,n,n+1,\dots,n+m.
\$
$V_i$ is a random variable that depends on the training process and the split of calibration and test data. The next lemma shows that under Assumption~\ref{assump:permute}, the non-conformity scores are still exchangeable.
\begin{lemma}\label{lem:invariant}

In the transductive setting described in Section~\ref{sec:background}, conditional on the entire unordered graph $(\cV,\cE)$, all the attribute and label information $\{(\mathbf{x}_v,y_v)\}_{v\in \cV}$, and the index sets $\cD_\train$ and $\cD_{\textrm{ct}} :=\cD_\calib\cup \cD_\test$,  \revise{the unordered set of the scores $[V_i]_{i=1}^{n+m}$ are fixed}.
Also, the calibration scores $\{V_i\}_{i=1}^n$ are a simple random sample from $\{V_i\}_{i=1}^{n+m}$. That is, for any subset $\{v_1,\dots,v_n\}\subseteq \{V_i\}_{i=1}^{n+m}$ of size $n$, 
$
P( \{V_i\}_{i=1}^n = \{v_1,\dots,v_n\} \biggiven \{V_i\}_{i=1}^{n+m}) 
= 1/ \textstyle{  \binom{n}{|\cD_\textrm{ct}|} }.
$ 
\end{lemma}
Based on Lemma~\ref{lem:invariant}, we next show  that any permutation-invariant non-conformity score leads to valid prediction sets, and provide an exact characterization of the distribution of test-time coverage.

\begin{theorem}
\label{thm:exch}
Given any score $V$ obeying Assumption~\ref{assump:permute} and any confidence level $\alpha\in(0,1)$, we define the split conformal prediction set as 
$
\hat{C}(x) = \{y\colon V(x,y) \leq \hat\eta\},
$
where 
\$
\hat\eta = \inf\Big\{\eta\colon \textstyle{\frac{1}{n} \sum_{i=1}^n} \ind\{ V(X_i,Y_i) \leq \eta \} \geq (1-\alpha) (1+1/n)\Big\}. 
\$
Then 
$\PP(Y_{n+j}\in \hat{C}(X_{n+j}))\geq 1-\alpha$, $\forall ~j=1,\dots,m$. Moreover, define $\hat{\mathrm{Cover}} = \frac{1}{m}\sum_{j=1}^m \ind\{Y_{n+j}\in \hat{C}(X_{n+j})\}$. 
If \revise{the $V_i$'s, $i\in \cD_\calib\cup\cD_\test$}, have no ties almost surely, then for any $t\in(0,1)$,
\$
\PP\big( \hat{\mathrm{Cover}} \leq t\big)
= 1 - \Phi_{\rm HG}\big(\lceil (n+1)(1-q)\rceil-1; m+n, n,  \lceil (1-q)(n+1)\rceil +     \lceil mt\rceil     \big),
\$
where $\Phi_{\rm HG}(\cdot;N,n,k)$ denotes the cumulative distribution function of a hyper-geometric distribution with parameters $N,n,k$ (drawing $k$ balls from an urn  wherein $n$ out of $N$ balls are white balls).
\end{theorem}

\begin{wrapfigure}{r}{0.26\textwidth} 
  \centering
  \vspace{-\intextsep}
    \includegraphics[width=\linewidth]{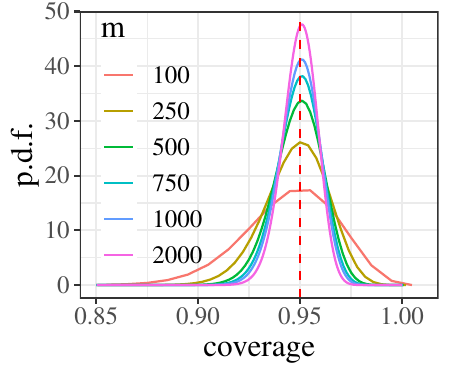}    \caption{\label{fig:cover_distr} P.d.f.~of  $\hat{\textrm{Cover}}$ for $n=1000$ and $\alpha=0.05$; curves represent different values of   test sample size $m$.}
\vspace{-4.2em}
\end{wrapfigure}

Figure~\ref{fig:cover_distr} plots the probability density functions (p.d.f.)~of $\hat{\textrm{Cover}}$ at a sequence of  $t\in[0,1]$ fixing $n=1000$ while varying $m$.  The exact distribution described in Theorem~\ref{thm:exch} is useful in determining the size of the calibration data in order for the test-time coverage to concentrate sufficiently tightly around $1-\alpha$. More discussion and visualization of $\widehat{\textrm{Cover}}$ under different values of $(n,m)$ are in Appendix~\ref{app:figures}. Note that similar exchangeability and validity results are obtained in several concurrent works~\cite{pmlr-v202-h-zargarbashi23a,lunde2023conformal}, yet without exact characterization of the test-time coverage. %

\section{\mname: Conformalized Graph Neural Networks}

We now propose a new method called \mname to reduce inefficiency while maintaining valid coverage. The key idea of \mname is to boost any given non-conformity score with graphical information. %
The method illustration is in Figure~\ref{fig:method} and pseudo-code is in Appendix~\ref{appendix:algorithm}. %

\xhdr{Efficiency-aware boosting} Standard CP takes any pre-trained predictor to construct the prediction set/interval (see Section~\ref{sec:background}). A key observation is that the training stage is not aware of the post-hoc stage of conformal prediction set/interval construction. Thus, it is not optimized for efficiency. Our high-level idea is to include an additional correction step that boosts the non-conformity score, which happens after the model training and before conformal prediction. To ensure flexibility and practicality, our framework is designed as a generic wrapper that works for any pre-trained GNN model, without changing the base model training process.

\xhdr{Motivation: Inefficiency correlation} 
Our approach to boosting the scores is based on exploiting the correlation among connected nodes. 
Since the connected nodes usually represent entities that interact in the real world, there can be strong correlation between them. To be more specific, it is well established in network science that prediction residuals are correlated along edges~\cite{jia2020residual}. Such a result implies a similar phenomenon for inefficiency: taking CQR for regression as an example, the prediction interval largely depends on the  residual of the true outcome from the predicted quantiles. Hence, the prediction interval lengths are also highly correlated for connected nodes. We empirically verify this intuition in Figure~\ref{fig:demo_residual}, where we plot the difference in the prediction interval lengths for connected/unconnected node pairs in the Anaheim dataset using vanilla CQR for GCN.\footnote{That is, we take the predicted quantiles from GCN and directly builds prediction intervals using CQR.} In Figure~\ref{fig:demo_residual}, we observe that inefficiency has a topological root: connected nodes usually have similar residual scales, suggesting the existence of rich neighborhood information for the residuals. This motivates us to utilize such information to correct the scores and achieve better efficiency.

\xhdr{Topology-aware correction model} Based on the motivation above, we update the model predictions using the neighbor predictions. However, the relationship in neighborhood predictions could be complex (\ie~beyond homophily), making heuristic aggregation such as averaging/summing/etc. overly simplistic and not generalizable to all types of graphs. Ideally, we want to design a general and powerful mechanism that  flexibly aggregates graph information. This requirement could be perfectly fulfilled by GNN message passing as it represents a class of learnable aggregation functions over graphs. Therefore, we use a separate GNN learner parameterized by $\vartheta$ for the same network $G$ but with modified input node features; specifically, we use the base GNN prediction ($\mathbf{X}_0 = \hat{\mu}(X)$) as input, and output $\Tilde{\mu}(X) = \mathrm{GNN}_\vartheta (\hat{\mu}(X), G)$. We will then use $\tilde\mu(X)$ as the input for constructing conformal prediction sets/intervals. Note that this second GNN model is a post-hoc process and only requires base GNN predictions, instead of access to the base GNN model.

\begin{figure}[t]
\begin{minipage}{.3\textwidth}
    \centering
    \includegraphics[width=0.9\textwidth]{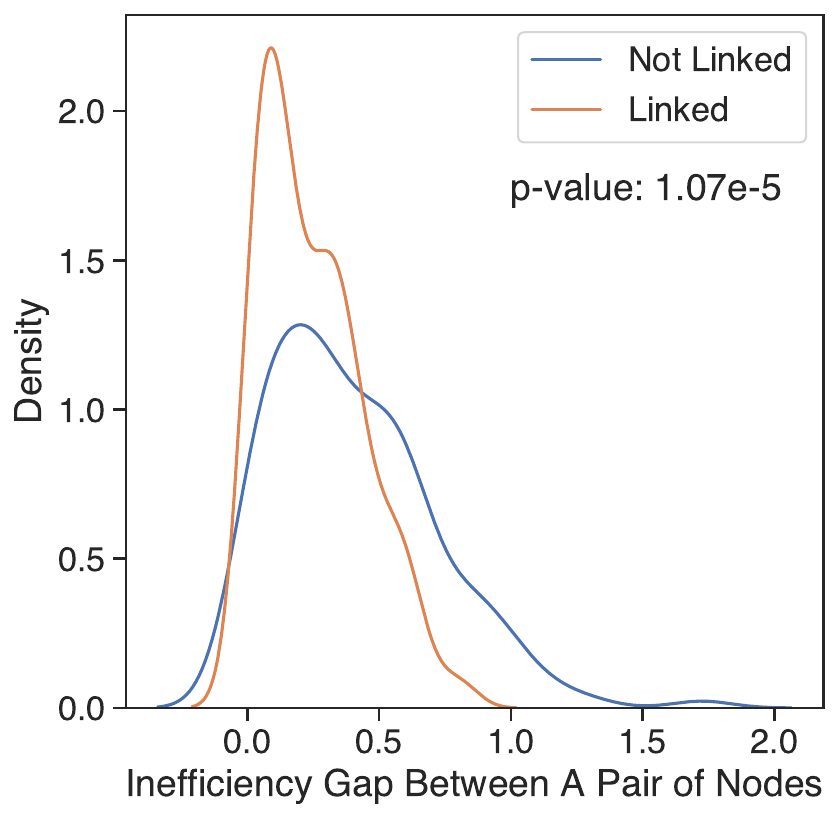}
    \captionof{figure}{Inefficiency is correlated in the network. Connected nodes have significantly smaller gaps in \revise{prediction interval length} compared to unconnected nodes.}
    \label{fig:demo_residual}
\end{minipage}
\hspace{2mm}
\begin{minipage}{.68\textwidth}
\centering
    \includegraphics[width=0.95\textwidth]{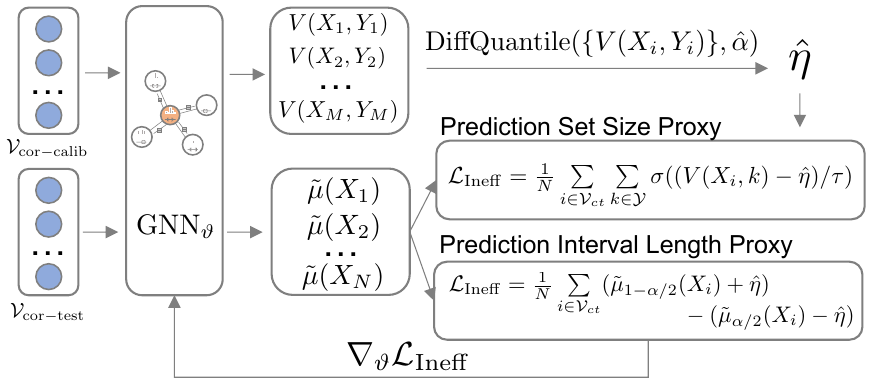}
    \captionof{figure}{We simulate the downstream conformal step and optimize for inefficiency directly. We first produce differentiable quantile $\hat{\eta}$ using $V(X_i, Y_i)$ from $\mathcal{V}_\mathrm{cor-cal}$. We then construct a prediction set size/interval length proxy on $\mathcal{V}_\mathrm{cor-test}$ and directly minimize inefficiency loss by updating $\mathrm{GNN}_\vartheta$.}
    \label{fig:eff}
\end{minipage}
\end{figure}

\xhdr{Training with conformal-aware inefficiency loss} Given the hypothesis class, it remains to devise a concrete recipe for training the correction model $\mathrm{GNN}_\vartheta$ parameters. Recall that as with many other prediction models, a GNN model is typically trained to optimize prediction loss (\ie cross-entropy loss or mean squared error) but not geared towards efficiency for the post-hoc conformal prediction step. We design $\mathrm{GNN}_\vartheta$ to be efficiency-aware by proposing a differentiable inefficiency loss that $\mathrm{GNN}_\vartheta$ optimizes over; this allows integration of GNN message passing to exploit the neighborhood information and also ensures a good $\tilde\mu(\cdot)$ that leads to efficient prediction sets in downstream steps.

We first withhold a small fraction $\gamma$ of the calibration dataset and use it for the correction procedure. The remaining data is used as the ``usual'' calibration data for building the conformal prediction set. We then further split the withheld data into a correction calibration set $\mathcal{V}_\mathrm{cor-cal}$ and correction testing set $\mathcal{V}_\mathrm{cor-test}$, to simulate the downstream conformal inference step. Given $\tilde{\mu}(X) = \mathrm{GNN}_\vartheta (\hat{\mu}(X), G)$ and a target miscoverage rate $\alpha$, the framework follows three steps:

(1) \underline{Differentiable quantile}: we compute a smooth quantile $\hat{\eta}$ based on the $\mathcal{V}_\mathrm{cor-cal}$ by 
\$
\hat{\eta} = \mathrm{DiffQuantile}(\{V(X_i, Y_i) | i \in \mathcal{V}_\mathrm{cor-cal}\}, (1-\alpha)(1+1/|\mathcal{V}_\mathrm{cor-cal}|).
\$ 
Since the non-conformity score is usually differentiable, it only requires differentiable quantile calculation where there are well-established methods available~\cite{chernozhukov2010quantile,blondel2020fast}. 

(2) \underline{Differentiable inefficiency proxy}: we then construct a differentiable proxy $\mathbf{c}$ of the inefficiency on $\mathcal{V}_\mathrm{cor-test}$ by simulating the downstream conformal prediction procedures. We propose general formulas to construct $\mathbf{c}$ that are applicable for any classification and regression tasks respectively:

\textit{a. Inefficiency loss instantiation for Classification}: The desirable proxy is to simulate the prediction set size using $\mathcal{D}_\mathrm{cor-test}$ as the ``test'' data and $\cD_{\mathrm{cor-calib}}$ as the ``calibration'' data. For class $k$ and node $i$ in $\mathcal{D}_\mathrm{cor-test}$, the non-conformity score is $V(X_i, k)$ for class $k$, where $V(\cdot)$, for instance, is the APS score in Section~\ref{sec:background}. Then, we define the inefficiency proxy as 
\$\mathbf{c}_{i,k} = \sigma(\frac{V(X_i, k) - \hat{\eta}}{\tau}),\$ 
where $\sigma(x) = \frac{1}{1+e^{-x}}$ is the sigmoid function and $\tau$ is a temperature hyper-parameter~\cite{stutz2021learning}. It can be interpreted as a soft assignment of class $k$ to the prediction set. When $\tau \rightarrow 0$, it becomes a hard assignment. The per-sample inefficiency proxy is then readily constructed as $\mathbf{c}_{i} = \frac{1}{|\mathcal{Y}|}\sum_{k\in\mathcal{Y}}\mathbf{c}_{i,k}$. 

\textit{b. Inefficiency loss instantiation for Regression}: The desirable proxy is to simulate the prediction interval length. For node $i$ in $\mathcal{V}_\mathrm{cor-test}$, the conformal prediction interval is $[\tilde{\mu}_{\alpha/2}(X_i) - \hat{\eta}, \tilde{\mu}_{1-\alpha/2}(X_i) + \hat{\eta}]$. Thus, the per-sample prediction interval length could be directly calculated as 
\$\mathbf{c}_i = (\tilde{\mu}_{1-\alpha/2}(X_i) + \hat{\eta}) - (\tilde{\mu}_{\alpha/2}(X_i) - \hat{\eta}).\$ 
Since $\mathrm{GNN}_\vartheta$ maps intervals to intervals and do not pose a constraint on the prediction, it may incur a trivial optimized solution where $\tilde{\mu}_{1-\alpha/2}(X) < \tilde{\mu}_{\alpha/2}(X)$. Thus, we pose an additional consistency regularization term: $(\tilde{\mu}_{1-\alpha/2}(X) - \hat{\mu}_{1-\alpha/2}(X))^2 + (\tilde{\mu}_{\alpha/2}(X) - \hat{\mu}_{\alpha/2}(X))^2$. This regularizes the updated intervals to not deviate significantly to reach the trivial solution. 

(3) \underline{Inefficiency loss}: finally, the inefficiency loss is an average of inefficiency proxies $L_\mathrm{ineff} = \frac{1}{|\mathcal{V}_\mathrm{cor-test}|} \sum_i \mathbf{c}_i$. The $\mathrm{GNN}_\vartheta$ is optimized using backpropagation in an end-to-end fashion. 

\xhdr{Conditional coverage} A natural question is whether optimizing the efficiency of conformal prediction may hurt its conditional validity.\footnote{Conditional coverage asks for $\PP(Y_{n+j}\in \hat{C}(X_{n+j})\given X_{n+j}=x)\approx 1-\alpha$ for all $x\in \cX$. Although exact conditional validity is statistically impossible~\cite{foygel2021limits}, approximate conditional validity is a practically important property that APS and CQR are designed for. See Section~\ref{sec:exp} for common ways to assess conditional coverage.} In Section~\ref{sec:exp}, we empirically demonstrate satisfactory conditional coverage across various graph features, which even improves upon the direct application of APS and CQR to graph data. We conjecture that it is because we correct for the correlation among nodes. However, theoretical understanding    is left for future investigation.

\xhdr{Graph exchangeability} The post-hoc correction model is $\mathrm{GNN}$-based, thus, it is permutation-invariant. Thus, it satisfies the exchangeability condition laid out in our theory in Section~\ref{sec:theory}. Empirically, we demonstrate in Section~\ref{sec:exp} that \mname achieves target empirical marginal coverage.

\xhdr{Computational cost} We remark that \mname scales similarly as base GNN training since the correction step follows a standard GNN training procedure but with a modified input attribute and loss function. Notably, as the input to the correction model usually has a smaller attribute size (the number of classes for classification and 2 for regression), it has smaller parameter size than standard GNN training. In addition, it is also compatible with standard GNN mini-batching techniques. 

\xhdr{General loss functions} Finally, we note that the choice of our loss function can be quite general. For instance,  one may directly optimize for conditional validity by choosing a proper loss function.  

\section{Experiments} \label{sec:exp}

We conduct experiments to demonstrate the advantages of \mname over other UQ methods in achieving empirical marginal coverage for graph data, as well as the efficiency improvement with \mname. We also evaluate conditional coverage of \mname  and conduct systematic ablation and parameter analysis to show the robustness of \mname.

\xhdr{Evaluation setup} For node classification, we follow a standard semi-supervised learning evaluation procedure~\cite{kipf2016semi}, where we randomly split data into folds with 20\%/10\%/70\% nodes as $\mathcal{D}_\mathrm{train}$/$\mathcal{D}_\mathrm{valid}$/$\mathcal{D}_\mathrm{calib} \cup \mathcal{D}_\mathrm{test}$. For the node regression task, we follow a previous evaluation procedure from~\cite{jia2020residual} and randomly split the data into folds with 50\%/10\%/40\% nodes as $\mathcal{D}_\mathrm{train}$/$\mathcal{D}_\mathrm{valid}$/$\mathcal{D}_\mathrm{calib} \cup \mathcal{D}_\mathrm{test}$. We conduct 100 random splits of calibration/testing sets to estimate the empirical coverage. Using the test-time coverage distribution in Figure~\ref{fig:cover_distr} to ensure that coverage is concentrated tightly around 1-$\alpha$, we modify the calibration set size to $\min\{1000,|\cD_\calib\cup\cD_\test|/2\}$, and use the rest as the test sample.
For a fair comparison, we first train 10 runs of the base GNN model and then fix the predictions (\ie~the input to UQ baselines and \mname). In this way, we ensure that the gain is not from randomness in base model training. The hyperparameter search strategy and configurations for \mname and baselines can be found in Appendix~\ref{appendix:hyperparam}.

\xhdr{Models \& baselines to evaluate coverage} %
For classification, we first use general statistical calibration approaches including temperate scaling~\cite{guo2017calibration}, vector scaling~\cite{guo2017calibration}, ensemble temperate scaling~\cite{zhang2020mix}. We also use SOTA GNN-specific calibration learners including CaGCN~\cite{wang2021confident} and GATS~\cite{hsu2022makes}. The prediction set is the set of classes from highest to lowest scores until accumulative scores exceed 1-$\alpha$. For regression, we construct prediction intervals using quantile regression (QR)~\cite{koenker2001quantile}, Monte Carlo dropouts (MC dropout)~\cite{gal2016dropout}, and Bayesian loss to model both aleatoric and epistemic uncertainty~\cite{kendall2017uncertainties}. More information about baselines can be found in Appendix~\ref{appendix:baseline}.

\xhdr{Models \& baselines to evaluate efficiency} As smaller coverage always leads to higher efficiency, for a fair comparison,  we can only compare methods on efficiency that achieve the same coverage. Thus, we do not evaluate UQ baselines here since they do not produce exact coverage and are thus not comparable. While any CP-based methods produce exact coverage, to the best of our knowledge, there are no existing graph-based conformal prediction methods for transductive settings. Thus, we can only compare with the direct application of conformal prediction (CP) to base GNN. In the main text, we only show results for GCN~\cite{kipf2016semi} as the base model; results of three additional popular GNN models (GraphSAGE~\cite{hamilton2017inductive}, GAT~\cite{velivckovic2017graph}, and SGC~\cite{wu2019simplifying}) are deferred to Appendix~\ref{appendix:gnn_architectures}. 

\xhdr{Datasets} We evaluate \mname on 8 node classification datasets and 7 node regression datasets with diverse network types such as geographical network, transportation network, social network, citation network, and financial network. Dataset statistics are in Appendix~\ref{appendix:dataset}.

\begingroup

\setlength{\tabcolsep}{2pt}

\begin{table}[t!]
    \centering
    \caption{Empirical marginal coverage of node classification(upper table) and node regression tasks(lower table). The result takes the average and standard deviation across 10 GNN runs with 100 calib/test splits. \greencheck~means that the UQ method reaches the target coverage (\ie~coverage $\ge$ 0.95) while \redmark~means that it fails to reach it. The last column "Covered" becomes \greencheck~if a UQ method reaches target coverage for all datasets and \redmark~vice versus. 
    }
    \vspace{0.5mm}
    \adjustbox{width=\textwidth}{
    \begin{tabular}{l|l|c|c|c|c|c|c|c|c|c}
    \toprule
    Task & UQ Model & Cora & DBLP & CiteSeer & PubMed &Computers & Photo & CS & Physics & Covered? \\ \midrule
    \multirow{6}{*}{\parbox{1.2cm}{Node classif.}} &Temp. Scale.&0.946\std{.003} \redmark&0.920\std{.009} \redmark&0.952\std{.004} \greencheck&0.899\std{.002} \redmark&0.929\std{.002} \redmark&0.962\std{.002} \greencheck&0.957\std{.001} \greencheck&0.969\std{.000} \greencheck & \redmark\\
&Vector Scale.&0.944\std{.004} \redmark&0.921\std{.009} \redmark&0.951\std{.004} \greencheck&0.899\std{.003} \redmark&0.932\std{.002} \redmark&0.963\std{.002} \greencheck&0.958\std{.001} \greencheck&0.969\std{.000} \greencheck& \redmark\\
&Ensemble TS&0.947\std{.003} \redmark&0.920\std{.008} \redmark&0.953\std{.003} \greencheck&0.899\std{.002} \redmark&0.930\std{.002} \redmark&0.964\std{.002} \greencheck&0.958\std{.001} \greencheck&0.969\std{.000} \greencheck& \redmark\\
&CaGCN&0.939\std{.005} \redmark&0.922\std{.004} \redmark&0.949\std{.005} \redmark&0.898\std{.003} \redmark&0.926\std{.003} \redmark&0.956\std{.002} \greencheck&0.954\std{.003} \greencheck&0.968\std{.001} \greencheck& \redmark\\
&GATS&0.939\std{.005} \redmark&0.921\std{.004} \redmark&0.951\std{.005} \greencheck&0.898\std{.002} \redmark&0.925\std{.002} \redmark&0.957\std{.002} \greencheck&0.957\std{.001} \greencheck&0.968\std{.000} \greencheck& \redmark\\\cmidrule{2-11}
&CF-GNN&0.952\std{.001} \greencheck&0.952\std{.001} \greencheck&0.953\std{.001} \greencheck&0.953\std{.001} \greencheck&0.952\std{.001} \greencheck&0.953\std{.001} \greencheck&0.952\std{.001} \greencheck&0.952\std{.001} \greencheck& \greencheck\\
     \bottomrule
    \end{tabular}
    }
    \newline
    \vspace*{1mm}

    \adjustbox{width=0.9\textwidth}{\begin{tabular}{l|l|c|c|c|c|c|c|c|c}
    \toprule
     Task & UQ Model & Anaheim & Chicago & Education & Election & Income & Unemploy. & Twitch & Covered?\\ \midrule
     \multirow{4}{*}{\parbox{1.2cm}{Node regress.}}&QR&0.943\std{.031} \redmark&0.950\std{.007} \redmark&0.959\std{.001} \greencheck&0.956\std{.004} \greencheck&0.960\std{.005} \greencheck&0.954\std{.004} \greencheck&0.900\std{.015} \redmark& \redmark\\
&MC dropout&0.553\std{.022} \redmark&0.427\std{.015} \redmark&0.423\std{.013} \redmark&0.417\std{.008} \redmark&0.532\std{.022} \redmark&0.489\std{.016} \redmark&0.448\std{.017} \redmark& \redmark\\
&BayesianNN&0.967\std{.001} \greencheck&0.955\std{.003} \greencheck&0.957\std{.002} \greencheck&0.958\std{.009} \greencheck&0.970\std{.004} \greencheck&0.960\std{.001} \greencheck&0.923\std{.006} \redmark& \redmark\\\cmidrule{2-10}
&CF-GNN&0.957\std{.003} \greencheck&0.954\std{.002} \greencheck&0.951\std{.001} \greencheck&0.950\std{.001} \greencheck&0.951\std{.001} \greencheck&0.951\std{.001} \greencheck&0.954\std{.001} \greencheck & \greencheck\\
     \bottomrule
    \end{tabular}
    }
    \label{tab:cov}
\end{table}

\endgroup

\subsection{Results}

\xhdr{\mname achieves empirical marginal coverage while existing UQ methods do not} We report marginal coverage of various UQ methods with target coverage at 95\% (Table~\ref{tab:cov}). There are three key takeaways. Firstly, none of these UQ methods  achieves the target coverage for all datasets while \mname does, highlighting the lack of statistical rigor in those methods and the necessity for a guaranteed UQ method. Secondly, it validates our theory from Section~\ref{sec:theory} that \mname achieves designated coverage in transductive GNN  predictions. Lastly, \mname achieves empirical coverage that is close to the target coverage while baseline UQ methods are not. This controllable feature of \mname is practically useful for practitioners that aim for a specified coverage in settings such as planning and selection.

\xhdr{\mname significantly reduces inefficiency} We report empirical inefficiency for 8 classification and 7 regression datasets (Table~\ref{tab:eff}). We observe that we achieve consistent improvement across datasets with up to 74\% reduction in the prediction set size/interval length. 
We additionally conduct the same experiments for 3 other GNN models including GAT, GraphSAGE, and SGC in Appendix~\ref{appendix:gnn_architectures} and we observe that performance gain is generalizable to diverse architecture choices. Furthermore, \mname yields more efficient prediction sets than existing UQ methods even if we manually adjust the nominal level of them to achieve $95\%$ empirical coverage (it is however impossible to do so in practice, here we do this for evaluation). For instance, the best calibration method GATS yields an average prediction size of 1.82 on Cora when the nominal level is tuned to achieve $95\%$ empirical coverage. In contrast, \mname has an average size of 1.76, smaller than GATS. In Appendix~\ref{appendix:raps}, we also observe that \mname also reduces inefficiency for advanced conformal predictor RAPS for classification task. In addition, we find that CF-GNN yields little changes to the prediction accuracy of the original GNN model (Appendix~\ref{appendix:accuracy_prediction}).

\begin{table}[t]
    \centering
    \caption{Empirical inefficiency measured by the size/length of the prediction set/interval for node classification (left table)/regression(right table). A smaller number has better efficiency. We show the relative improvement (\%) of \mname over CP on top of the $\rightarrow$. The result uses APS for classification and CQR for regression with GCN as the base model. Additional results on other GNN models are at Appendix~\ref{appendix:gnn_architectures}. We report the average and standard deviation of prediction sizes/lengths calculated from 10 GNN runs, each with 100 calibration/test splits. }
    \vspace{1mm}
    \begin{adjustbox}{width=0.38\textwidth}
    \begin{tabular}{l|l|c}
    \toprule
    Task & Dataset & CP $\xrightarrow{\hspace*{0.7cm}}$CF-GNN \\\midrule
    \multirow{8}{*}{\parbox{1.2cm}{Node classif.}} &Cora&3.80\std{.28}$\xrightarrow{-53.61\%}$1.76\std{.27}\\
    &DBLP&2.43\std{.03}$\xrightarrow{-49.13\%}$1.23\std{.01}\\
    &CiteSeer&3.86\std{.11}$\xrightarrow{-74.27\%}$0.99\std{.02}\\
    &PubMed&1.60\std{.02}$\xrightarrow{-19.05\%}$1.29\std{.03}\\
    &Computers&3.56\std{.13}$\xrightarrow{-49.05\%}$1.81\std{.12}\\
    &Photo&3.79\std{.13}$\xrightarrow{-56.28\%}$1.66\std{.21}\\
    &CS&7.79\std{.29}$\xrightarrow{-62.16\%}$2.95\std{.49}\\
    &Physics&3.11\std{.07}$\xrightarrow{-62.81\%}$1.16\std{.13}\\
\midrule
    \multicolumn{2}{l|}{Average Improvement} & -53.75\% \\
     \bottomrule
    \end{tabular}
    \end{adjustbox}
    \qquad
    \begin{adjustbox}{width=0.4\textwidth}
    \begin{tabular}{l|l|c}
    \toprule
    Task & Dataset & CP $\xrightarrow{\hspace*{0.7cm}}$CF-GNN \\\midrule
    \multirow{7}{*}{\parbox{1.2cm}{Node regress.}} &Anaheim&2.89\std{.39}$\xrightarrow{-25.00\%}$2.17\std{.11}\\
    &Chicago&2.05\std{.07}$\xrightarrow{-0.48\%}$2.04\std{.17}\\
    &Education&2.56\std{.02}$\xrightarrow{-5.07\%}$2.43\std{.05}\\
    &Election&0.90\std{.01}$\xrightarrow{+0.21\%}$0.90\std{.02}\\
    &Income&2.51\std{.12}$\xrightarrow{-4.58\%}$2.40\std{.05}\\
    &Unemploy&2.72\std{.03}$\xrightarrow{-10.83\%}$2.43\std{.04}\\
    &Twitch&2.43\std{.10}$\xrightarrow{-1.36\%}$2.39\std{.07}\\ \midrule
    \multicolumn{2}{l|}{Average Improvement} & -6.73\% \\
     \bottomrule
    \end{tabular}
    \end{adjustbox}

    \label{tab:eff}
\end{table}

\begin{table}[t]
\begin{minipage}[c]{0.6\textwidth}
\centering
    \captionof{table}{\mname achieves conditional coverage, measured by Worse-Slice Coverage~\cite{romano2020classification}. We use Cora/Twitch as an example classification/regression dataset. Results on other network features and results on target coverage of 0.9 can be found in Appendix~\ref{appendix:cond_cov_networks}.}
    \adjustbox{max width=0.95\textwidth}{
    \begin{tabular}{l|cc|cc}
    \toprule
    Target: 0.95  & \multicolumn{2}{c|}{Classification} & \multicolumn{2}{c}{Regression} \\ \midrule
    Model & CP & CF-GNN & CP & CF-GNN  \\ \midrule
    Marginal Cov. & 0.95\std{.01} & 0.95\std{.01} & 0.96\std{.02} & 0.96\std{.02}\\ \midrule
    Cond. Cov. (Input Feat.) & 0.94\std{.02} & 0.94\std{.03} & 0.95\std{.04} & 0.94\std{.05} \\ \midrule
    Cond. Cov. (Cluster) & 0.89\std{.06} & 0.93\std{.04} & 0.96\std{.03} & 0.96\std{.03} \\
    Cond. Cov. (Between) & 0.81\std{.06} & 0.95\std{.03} & 0.94\std{.05} & 0.94\std{.05} \\
    Cond. Cov. (PageRank) & 0.78\std{.06} & 0.94\std{.03} & 0.94\std{.05} & 0.94\std{.05}\\
    \bottomrule
    \end{tabular}
    \label{tab:conditional_coverage}
    }
\end{minipage}
\hspace{3mm}
\begin{minipage}[c]{0.36\textwidth}
\centering
    \captionof{table}{Ablation. For Size/length, we use Cora/Anaheim dataset with GCN backbone. Each experiment is with 10 independent base model runs with 100 conformal split runs.}
    \adjustbox{max width=0.95\textwidth}{
    \begin{tabular}{cc|cc}
    \toprule
    Topology & Ineff. & \multirow{2}{*}{Size} & \multirow{2}{*}{Length} \\ 
    -aware & Loss & & \\ \midrule
    \greencheck & \greencheck & 1.76\std{.27} &  2.17\std{.11} \\ \midrule
    \greencheck  & \redmark & 2.42\std{.35} & 2.23\std{.10}\\ 
    \redmark & \greencheck & 2.35\std{.47} & 2.32\std{.18}\\
    \redmark & \redmark & 3.80\std{.28} & 2.89\std{.39} \\ 
    \bottomrule
    \end{tabular}
    }
    \label{tab:ablation}
\end{minipage}

\end{table}

\xhdr{\mname empirically maintains conditional coverage} While \mname achieves marginal coverage, it is highly desirable to have a method that achieves reasonable conditional coverage, which was the motivation of APS and CQR. We follow~\cite{romano2020classification} to evaluate conditional coverage via the Worst-Slice (WS) coverage, which takes the worst coverage across slices in the feature space (\ie~node input features). We observe that \mname achieves a WS coverage close to $1-\alpha$, indicating satisfactory conditional coverage (Cond. Cov. (Input Feat.) row in Table~\ref{tab:conditional_coverage}). Besides the raw features, for each node, we also construct several network features (which are label agnostic) including clustering coefficients, betweenness, PageRank, closeness, load, and harmonic centrality, and then calculate the WS coverage over the network feature space. We observe close to 95\% WS coverage for various network features, suggesting \mname also achieves robust conditional coverage over network properties. We also see that the direct application of CP (\ie~without graph correction) has much smaller WS coverage for classification, suggesting that adjusting for neigborhood information in \mname implicitly improves conditional coverage.

\xhdr{Ablation} We conduct ablations in Table~\ref{tab:ablation} to test two main components in \mname, topology-aware correction model, and inefficiency loss. We first remove the inefficiency loss and replace it with standard prediction loss. The performance drops as expected, showing the power of directly modeling inefficiency loss in the correction step. 
Secondly, we replace the GNN correction model with an MLP correction model. The performance drops significantly, showing the importance of the design choice of correction model  and justifying our motivation on inefficiency correlation over networks.    

\xhdr{Parameter analysis} We conduct additional parameter analysis to test the robustness of \mname. We first adjust the target coverage rate and calculate the inefficiency (Figure~\ref{fig:param}(1)). \mname consistently beats the vanilla CP across all target coverages. Moreover, we adjust the fraction $\gamma$ of the holdout calibration data in building the inefficiency loss, and observe that \mname achieves consistent improvement in inefficiency (Figure~\ref{fig:param}(2)). We also observe a small fraction (10\%) leads to excellent performance, showing that our model only requires a small amount of data for the inefficiency loss and leaves the majority of the calibration data for downstream conformal prediction.

\begin{figure}[t]
    \centering
    \includegraphics[width=\textwidth]{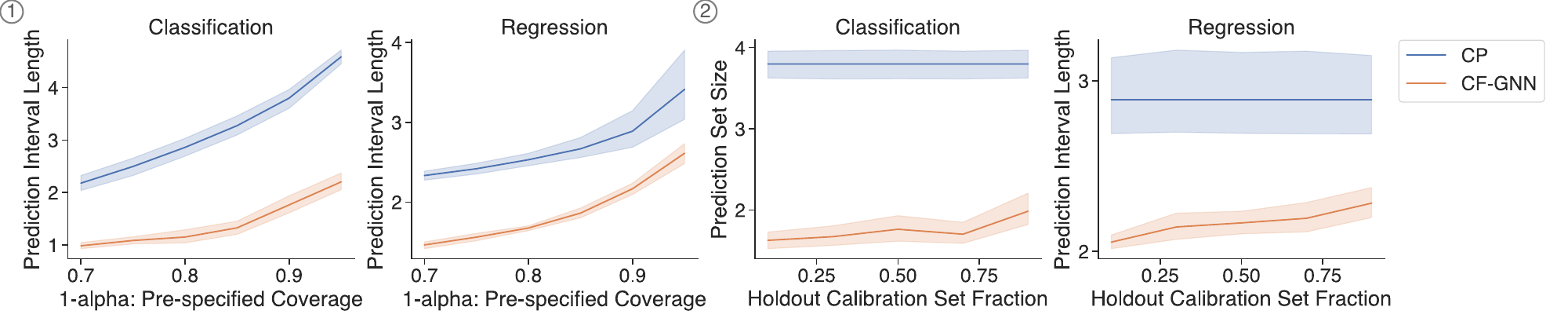}
    \caption{(1) Parameter analysis on inefficiency given different target coverage rate 1-$\alpha$. (2) Parameter analysis on inefficiency given calibration set holdout fraction. Analyses use Cora/Anaheim for classification/regression. }
    \label{fig:param}
\end{figure}

\section{Related Works}
\label{sec:related_work}

We discuss here related works that are closest to the ideas in \mname and provide extended discussion on other related works in Appendix~\ref{appendix:related_work}.

\underline{(1) Uncertainty quantification (UQ) for GNN:} Many UQ methods have been proposed to construct model-agnostic uncertain estimates for both classification~\cite{guo2017calibration,zhang2020mix,gupta2020calibration,kull2019beyond,abdar2021review} and regression~\cite{koenker2001quantile,takeuchi2006nonparametric,sheather1990kernel,gal2016dropout,lakshminarayanan2017simple,kuleshov2018accurate,ovadia2019can,kendall2017uncertainties,izmailov2021bayesian}. Recently, specialized calibration methods for GNNs that leverage network principles such as homophily have been developed~\cite{wang2021confident,hsu2022makes}. However, these UQ methods can fail to provide a statistically rigorous and empirically valid coverage guarantee (see Table~\ref{tab:cov}). In contrast, \mname achieves valid marginal coverage in both theory and practice. 

\underline{(2) Conformal prediction for GNN:} The application of CP to graph-structured data remains largely unexplored. \cite{clarkson2022distribution} claims that nodes in the graph are not exchangeable in the inductive setting and employs the framework of~\cite{barber2022conformal} to construct prediction sets using neighborhood nodes as the calibration data for mitigating the miscoverage due to non-exchangeability. In contrast, we study the transductive setting where certain exchangeability property holds; thus, the method from~\cite{barber2022conformal} are not comparable to ours. Concurrent with our work, \cite{pmlr-v202-h-zargarbashi23a} studies the exchangeability under transductive setting and proposes a diffusion-based method for improving efficiency, which can be viewed as an instantiation of our approach where the GNN correction learns an identity mapping; \cite{lunde2023conformal} studies  exchangeability  in network regression for of non-conformity scores based on various network structures, with similar observations as our Theorem~\ref{thm:exch}. Other recent efforts in conformal prediction for graphs include~\cite{lunde2023validity,marandon2023conformal} which focus on distinct  problem settings. 

\underline{(3) Efficiency of conformal prediction:} How to achieve desirable properties beyond validity is an active topic in the CP community; we focus on the efficiency aspect here. One line of work designs ``good'' non-conformity scores in theory such as APS~\cite{romano2020classification} and CQR~\cite{romano2019conformalized}. More recent works take another approach, by modifying the training process of the prediction model.
\mname falls into the latter case, although our idea applies to any non-conformity score. ConfTr~\cite{stutz2021learning} also modifies training for improved efficiency. Our approach differs from theirs in significant ways. First, we develop a theory on CP validity on the graph data and leverage topological principles that are specialized to graph to improve efficiency while ConfTr focuses on i.i.d. vision image data. Also, ConfTr happens during base model training using the training set, while \mname conducts post-hoc correction using withheld calibration data without assuming access to base model training, making ConfTr not comparable to us. Finally, we also propose a novel loss for efficiency in regression tasks.

\section{Conclusion}

In this work, we extend conformal prediction to GNNs by laying out the theoretical conditions for finite-sample validity and proposing a flexible graph-based CP framework to improve efficiency. Potential directions for future work  include generalizing the inefficiency loss to other desirable CP properties such as robustness and conditional coverage; extensions to inductive settings or transductive but non-random split settings; extensions to other graph tasks such as link prediction, community detection, and so on.

\section{Acknowledgements}
K.H. and J.L. gratefully acknowledge the support of
DARPA under Nos. HR00112190039 (TAMI), N660011924033 (MCS);
ARO under Nos. W911NF-16-1-0342 (MURI), W911NF-16-1-0171 (DURIP);
NSF under Nos. OAC-1835598 (CINES), OAC-1934578 (HDR), CCF-1918940 (Expeditions), 
NIH under No. 3U54HG010426-04S1 (HuBMAP),
Stanford Data Science Initiative, 
Wu Tsai Neurosciences Institute,
Amazon, Docomo, GSK, Hitachi, Intel, JPMorgan Chase, Juniper Networks, KDDI, NEC, and Toshiba. The content is solely the responsibility of the authors and does not necessarily represent the official views of the funding entities.

\bibliographystyle{plain}
\bibliography{refs}

\begin{thebibliography}{10}

\bibitem{abdar2021review}
Moloud Abdar, Farhad Pourpanah, Sadiq Hussain, Dana Rezazadegan, Li~Liu,
  Mohammad Ghavamzadeh, Paul Fieguth, Xiaochun Cao, Abbas Khosravi, U~Rajendra
  Acharya, et~al.
\newblock A review of uncertainty quantification in deep learning: Techniques,
  applications and challenges.
\newblock {\em Information Fusion}, 76:243--297, 2021.

\bibitem{angelopoulos2020uncertainty}
Anastasios Angelopoulos, Stephen Bates, Jitendra Malik, and Michael~I Jordan.
\newblock Uncertainty sets for image classifiers using conformal prediction.
\newblock {\em ICLR}, 2021.

\bibitem{barber2022conformal}
Rina~Foygel Barber, Emmanuel~J Cand{\`e}s, Aaditya Ramdas, and Ryan~J
  Tibshirani.
\newblock Conformal prediction beyond exchangeability.
\newblock {\em arXiv:2202.13415}, 2022.

\bibitem{bates2021distribution}
Stephen Bates, Anastasios Angelopoulos, Lihua Lei, Jitendra Malik, and Michael
  Jordan.
\newblock Distribution-free, risk-controlling prediction sets.
\newblock {\em Journal of the ACM (JACM)}, 68(6):1--34, 2021.

\bibitem{blondel2020fast}
Mathieu Blondel, Olivier Teboul, Quentin Berthet, and Josip Djolonga.
\newblock Fast differentiable sorting and ranking.
\newblock In {\em ICML}, pages 950--959. PMLR, 2020.

\bibitem{chernozhukov2010quantile}
Victor Chernozhukov, Iv{\'a}n Fern{\'a}ndez-Val, and Alfred Galichon.
\newblock Quantile and probability curves without crossing.
\newblock {\em Econometrica}, 78(3):1093--1125, 2010.

\bibitem{clarkson2022distribution}
Jase Clarkson.
\newblock Distribution free prediction sets for node classification.
\newblock {\em LoG}, 2022.

\bibitem{foygel2021limits}
Rina Foygel~Barber, Emmanuel~J Cand{\`e}s, Aaditya Ramdas, and Ryan~J
  Tibshirani.
\newblock The limits of distribution-free conditional predictive inference.
\newblock {\em Information and Inference: A Journal of the IMA},
  10(2):455--482, 2021.

\bibitem{gal2016dropout}
Yarin Gal and Zoubin Ghahramani.
\newblock Dropout as a bayesian approximation: Representing model uncertainty
  in deep learning.
\newblock In {\em ICML}, pages 1050--1059. PMLR, 2016.

\bibitem{gao2023topology}
Jiayi Gao, Jiaxing Li, Ke~Zhang, and Youyong Kong.
\newblock Topology uncertainty modeling for imbalanced node classification on
  graphs.
\newblock In {\em ICASSP 2023-2023 IEEE International Conference on Acoustics,
  Speech and Signal Processing (ICASSP)}, pages 1--5. IEEE, 2023.

\bibitem{gibbs2021adaptive}
Isaac Gibbs and Emmanuel Cand{\`e}s.
\newblock Adaptive conformal inference under distribution shift.
\newblock {\em NeurIPS}, 34, 2021.

\bibitem{gilmer2017neural}
Justin Gilmer, Samuel~S Schoenholz, Patrick~F Riley, Oriol Vinyals, and
  George~E Dahl.
\newblock Neural message passing for quantum chemistry.
\newblock In {\em ICML}, pages 1263--1272. JMLR. org, 2017.

\bibitem{guo2017calibration}
Chuan Guo, Geoff Pleiss, Yu~Sun, and Kilian~Q Weinberger.
\newblock On calibration of modern neural networks.
\newblock In {\em ICML}, pages 1321--1330. PMLR, 2017.

\bibitem{gupta2020calibration}
Kartik Gupta, Amir Rahimi, Thalaiyasingam Ajanthan, Thomas Mensink, Cristian
  Sminchisescu, and Richard Hartley.
\newblock Calibration of neural networks using splines.
\newblock {\em ICLR}, 2021.

\bibitem{pmlr-v202-h-zargarbashi23a}
Soroush H.~Zargarbashi, Simone Antonelli, and Aleksandar Bojchevski.
\newblock Conformal prediction sets for graph neural networks.
\newblock In {\em ICML}, 2023.

\bibitem{hamilton2017inductive}
Will Hamilton, Zhitao Ying, and Jure Leskovec.
\newblock Inductive representation learning on large graphs.
\newblock {\em NeurIPS}, 30, 2017.

\bibitem{hsu2022makes}
Hans Hao-Hsun Hsu, Yuesong Shen, Christian Tomani, and Daniel Cremers.
\newblock What makes graph neural networks miscalibrated?
\newblock {\em NeurIPS}, 2022.

\bibitem{ishimtsev2017conformal}
Vladislav Ishimtsev, Alexander Bernstein, Evgeny Burnaev, and Ivan Nazarov.
\newblock Conformal $ k $-nn anomaly detector for univariate data streams.
\newblock In {\em Conformal and Probabilistic Prediction and Applications},
  pages 213--227. PMLR, 2017.

\bibitem{izmailov2021bayesian}
Pavel Izmailov, Sharad Vikram, Matthew~D Hoffman, and Andrew Gordon~Gordon
  Wilson.
\newblock What are bayesian neural network posteriors really like?
\newblock In {\em ICML}, pages 4629--4640. PMLR, 2021.

\bibitem{jia2020residual}
Junteng Jia and Austion~R Benson.
\newblock Residual correlation in graph neural network regression.
\newblock In {\em KDD}, pages 588--598, 2020.

\bibitem{jin2022selection}
Ying Jin and Emmanuel~J Cand{\`e}s.
\newblock Selection by prediction with conformal p-values.
\newblock {\em arXiv:2210.01408}, 2022.

\bibitem{jin2023sensitivity}
Ying Jin, Zhimei Ren, and Emmanuel~J Cand{\`e}s.
\newblock Sensitivity analysis of individual treatment effects: A robust
  conformal inference approach.
\newblock {\em Proceedings of the National Academy of Sciences},
  120(6):e2214889120, 2023.

\bibitem{kendall2017uncertainties}
Alex Kendall and Yarin Gal.
\newblock What uncertainties do we need in bayesian deep learning for computer
  vision?
\newblock {\em NeurIPS}, 30, 2017.

\bibitem{kipf2016semi}
Thomas~N Kipf and Max Welling.
\newblock Semi-supervised classification with graph convolutional networks.
\newblock {\em ICLR}, 2017.

\bibitem{koenker2001quantile}
Roger Koenker and Kevin~F Hallock.
\newblock Quantile regression.
\newblock {\em Journal of Economic Perspectives}, 15(4):143--156, 2001.

\bibitem{kuleshov2018accurate}
Volodymyr Kuleshov, Nathan Fenner, and Stefano Ermon.
\newblock Accurate uncertainties for deep learning using calibrated regression.
\newblock In {\em ICML}, pages 2796--2804. PMLR, 2018.

\bibitem{kull2019beyond}
Meelis Kull, Miquel Perello~Nieto, Markus K{\"a}ngsepp, Telmo Silva~Filho, Hao
  Song, and Peter Flach.
\newblock Beyond temperature scaling: Obtaining well-calibrated multi-class
  probabilities with dirichlet calibration.
\newblock {\em NeurIPS}, 32, 2019.

\bibitem{lakshminarayanan2017simple}
Balaji Lakshminarayanan, Alexander Pritzel, and Charles Blundell.
\newblock Simple and scalable predictive uncertainty estimation using deep
  ensembles.
\newblock {\em NeurIPS}, 30, 2017.

\bibitem{lam2022graphcast}
Remi Lam, Alvaro Sanchez-Gonzalez, Matthew Willson, Peter Wirnsberger, Meire
  Fortunato, Alexander Pritzel, Suman Ravuri, Timo Ewalds, Ferran Alet, Zach
  Eaton-Rosen, et~al.
\newblock Graphcast: Learning skillful medium-range global weather forecasting.
\newblock {\em arXiv:2212.12794}, 2022.

\bibitem{lei2021conformal}
Lihua Lei and Emmanuel~J Cand{\`e}s.
\newblock Conformal inference of counterfactuals and individual treatment
  effects.
\newblock {\em Journal of the Royal Statistical Society: Series B (Statistical
  Methodology)}, 2021.

\bibitem{li2022graph}
Michelle~M Li, Kexin Huang, and Marinka Zitnik.
\newblock Graph representation learning in biomedicine and healthcare.
\newblock {\em Nature Biomedical Engineering}, pages 1--17, 2022.

\bibitem{lunde2023validity}
Robert Lunde.
\newblock On the validity of conformal prediction for network data under
  non-uniform sampling.
\newblock {\em arXiv preprint arXiv:2306.07252}, 2023.

\bibitem{lunde2023conformal}
Robert Lunde, Elizaveta Levina, and Ji~Zhu.
\newblock Conformal prediction for network-assisted regression.
\newblock {\em arXiv:2302.10095}, 2023.

\bibitem{marandon2023conformal}
Ariane Marandon.
\newblock Conformal link prediction to control the error rate.
\newblock {\em arXiv preprint arXiv:2306.14693}, 2023.

\bibitem{ovadia2019can}
Yaniv Ovadia, Emily Fertig, Jie Ren, Zachary Nado, David Sculley, Sebastian
  Nowozin, Joshua Dillon, Balaji Lakshminarayanan, and Jasper Snoek.
\newblock Can you trust your model's uncertainty? evaluating predictive
  uncertainty under dataset shift.
\newblock {\em NeurIPS}, 32, 2019.

\bibitem{romano2019conformalized}
Yaniv Romano, Evan Patterson, and Emmanuel Cand{\`e}s.
\newblock Conformalized quantile regression.
\newblock {\em NeurIPS}, 32, 2019.

\bibitem{romano2020classification}
Yaniv Romano, Matteo Sesia, and Emmanuel Cand{\`e}s.
\newblock Classification with valid and adaptive coverage.
\newblock {\em NeurIPS}, 33:3581--3591, 2020.

\bibitem{sheather1990kernel}
Simon~J Sheather and James~Stephen Marron.
\newblock Kernel quantile estimators.
\newblock {\em Journal of the American Statistical Association},
  85(410):410--416, 1990.

\bibitem{stadler2021graph}
Maximilian Stadler, Bertrand Charpentier, Simon Geisler, Daniel Z{\"u}gner, and
  Stephan G{\"u}nnemann.
\newblock Graph posterior network: Bayesian predictive uncertainty for node
  classification.
\newblock {\em Advances in Neural Information Processing Systems},
  34:18033--18048, 2021.

\bibitem{stutz2021learning}
David Stutz, Ali~Taylan Cemgil, Arnaud Doucet, et~al.
\newblock Learning optimal conformal classifiers.
\newblock {\em ICLR}, 2022.

\bibitem{takeuchi2006nonparametric}
Ichiro Takeuchi, Quoc Le, Timothy Sears, Alexander Smola, et~al.
\newblock Nonparametric quantile estimation.
\newblock 2006.

\bibitem{velivckovic2017graph}
Petar Veli{\v{c}}kovi{\'c}, Guillem Cucurull, Arantxa Casanova, Adriana Romero,
  Pietro Lio, and Yoshua Bengio.
\newblock Graph attention networks.
\newblock {\em ICLR}, 2018.

\bibitem{vovk2005algorithmic}
Vladimir Vovk, Alexander Gammerman, and Glenn Shafer.
\newblock {\em Algorithmic learning in a random world}.
\newblock Springer Science \& Business Media, 2005.

\bibitem{wang2021confident}
Xiao Wang, Hongrui Liu, Chuan Shi, and Cheng Yang.
\newblock Be confident! towards trustworthy graph neural networks via
  confidence calibration.
\newblock {\em NeurIPS}, 34:23768--23779, 2021.

\bibitem{wu2019simplifying}
Felix Wu, Amauri Souza, Tianyi Zhang, Christopher Fifty, Tao Yu, and Kilian
  Weinberger.
\newblock Simplifying graph convolutional networks.
\newblock In {\em ICML}, pages 6861--6871. PMLR, 2019.

\bibitem{wu2022graph}
Shiwen Wu, Fei Sun, Wentao Zhang, Xu~Xie, and Bin Cui.
\newblock Graph neural networks in recommender systems: a survey.
\newblock {\em ACM Computing Surveys}, 55(5):1--37, 2022.

\bibitem{yin2022conformal}
Mingzhang Yin, Claudia Shi, Yixin Wang, and David~M Blei.
\newblock Conformal sensitivity analysis for individual treatment effects.
\newblock {\em Journal of the American Statistical Association}, pages 1--14,
  2022.

\bibitem{zaffran2022adaptive}
Margaux Zaffran, Olivier F{\'e}ron, Yannig Goude, Julie Josse, and Aymeric
  Dieuleveut.
\newblock Adaptive conformal predictions for time series.
\newblock In {\em ICML}, pages 25834--25866. PMLR, 2022.

\bibitem{zhang2020mix}
Jize Zhang, Bhavya Kailkhura, and T~Yong-Jin Han.
\newblock Mix-n-match: Ensemble and compositional methods for uncertainty
  calibration in deep learning.
\newblock In {\em ICML}, pages 11117--11128. PMLR, 2020.

\bibitem{zhao2020uncertainty}
Xujiang Zhao, Feng Chen, Shu Hu, and Jin-Hee Cho.
\newblock Uncertainty aware semi-supervised learning on graph data.
\newblock {\em Advances in Neural Information Processing Systems},
  33:12827--12836, 2020.

\end{thebibliography}

\clearpage

\appendix

\section{Deferred details for Section~\ref{sec:theory}}
\label{app:theory}

\subsection{Technical proofs for theoretical results}
\label{app:proof}

\begin{proof}[Proof of Lemma~\ref{lem:invariant}]
Hereafter, we condition on the entire unordered graph, all the attribute and label information, and the index sets $\cD_\train$ and $\cD_{\textrm{ct}}$. 
We define the scores evaluated at the original node indices as 
\$
v_v =  V(\mathbf{x}_v,y_v; \{z_i\}_{i\in\cD_{\train}\cup\cD_{\textrm{valid}}}, \{\mathbf{x}_v\}_{v\in \cD_\calib\cup\cD_\test}, \cV, \cE),\quad v\in \cD_\calib\cup \cD_\test \subseteq \cV.
\$
By Assumption~\ref{assump:permute}, for any permutation $\pi$ of  $\cD_{\textrm{ct}}$, we always have
\$
v_v =  V(\mathbf{x}_v,y_v; \{z_i\}_{i\in\cD_{\train}\cup\cD_{\textrm{valid}}}, \{\mathbf{x}_{\pi(v)}\}_{v\in \cD_\calib\cup\cD_\test}, \cV_\pi, \cE_{\pi}).
\$
That is, 
given $\cD_{\textrm{ct}}$, the evaluated score at any $v\in \cD_{\textrm{ct}}$ remains invariant no matter which subset of $\cD_{\textrm{ct}}$ is designated as $\cD_\calib$. 
This implies that the scores are fixed after conditioning:
\$
[V_1, \dots,V_{n+m}] = [ v_v]_{v\in \cD_{\textrm{ct}}},
\$
where we use $[]$ to emphasize  unordered sets. 
Thus, the calibration scores $\{V_i\}_{i=1}^n$ is a subset of size $n$ of $ [ v_v]_{v\in \cD_{\textrm{ct}}}$.
Note that under random splitting in the transductive setting, any permutation $\pi$ of $\cD_{\textrm{ct}}$ occurs with the same probability, which gives the conditional probability in Lemma~\ref{lem:invariant}.
\end{proof}

\begin{proof}[Proof of Theorem~\ref{thm:exch}] 
Throughout this proof, we condition on the entire unordered graph, all the attribute and label information, and the index sets $\cD_\train$ and $\cD_{\textrm{ct}}$. 
By Lemma~\ref{lem:invariant}, after the conditioning, the unordered set of 
$\{V_i\}_{i=1}^{n+m}$ is fixed as $[v_v]_{v\in \cD_{\textrm{ct}}}$, and  $\{V_i\}_{i=1}^n$ is a simple random sample from $[v_v]_{v\in \cD_{\textrm{ct}}}$. As a result, any test sample $V(X_{n+j},Y_{n+j})$, $j=1,\dots,m$ is exchangeable with $\{V_i\}_{i=1}^n$. 
By standard theory for conformal prediction~\cite{vovk2005algorithmic}, this ensures \emph{valid marginal coverage}, i.e., 
$\PP(Y_{n+j}\in \hat{C}(X_{n+j}))\geq 1-\alpha$,  
where the expectation is over all the randomness. 
 
We now proceed to analyze the distribution of $\hat{\mathrm{Cover}}$. For notational convenience, we write $N=m+n$, and view $\cD_{\textrm{ct}}$ as the `population'. 
In this way, $\{V_i\}_{i=1}^n$ is a simple random sample from $[v_v]_{v\in \cD_{\textrm{ct}}}$. 
For every $\eta\in \RR$, we define the `population' cumulative distribution function (c.d.f.)
\$
F(\eta) = \frac{1}{N}\sum_{v\in \cD_{\textrm{ct}}} \ind\{v_v\leq \eta\},
\$
which is a deterministic function. 
We also define the calibration c.d.f.~as 
\$
\hat{F}_n(\eta) = \frac{1}{n}\sum_{v\in \cD_\calib} \ind\{v_v\leq \eta\} = \frac{1}{n}\sum_{i=1}^n \ind\{V_i\leq \eta\},
\$
which is random, and its randomness comes from which subset of $\cD_{\textrm{ct}}$ is $\cD_\calib$. 
By definition, 
\$
\hat\eta = \inf\{ \eta\colon \hat{F}_n(\eta) \geq (1-q)(1+1/n)  \}.
\$
Since the scores have no ties, we know 
\$
\hat{F}_n(\hat\eta) = \lceil (1-q)(n+1)\rceil /n.
\$
The test-time coverage can be written as 
\$
\hat{\rm Cover} & = \frac{1}{m}\sum_{j=1}^{m} \ind\{V_{n+j}\leq \hat\eta\} \\
&= \frac{1}{N-n}\bigg(\sum_{v\in \cD_{\textrm{ct}} } \ind\{v_v\leq \hat\eta) - \sum_{v\in \cD_\calib} \ind\{v_v\leq \hat\eta) \bigg)\\
&= \frac{N}{N-n} F(\hat\eta) - \frac{n}{N-n} \hat{F}_n(\hat\eta) 
= \frac{N}{N-n} F(\hat\eta) - \frac{\lceil (1-q)(n+1)\rceil}{N-n}.
\$
Now we characterize the distribution of $\hat\eta$. 
For any $\eta\in \RR$, by the definition of $\hat\eta$, 
\$
\PP(\hat\eta \leq \eta ) = \PP\big( n \hat{F}_n(\eta) \geq (n+1)(1-q)\big)
= \PP\big( n \hat{F}_n(\eta) \geq \lceil (n+1)(1-q)\rceil \big).
\$
Note that $n\hat{F}_n(\eta) = \sum_{v\in \cD_\calib} \ind\{v_v\leq \eta\}$ is the count of 
data in $\cD_\calib$ such that the score is below $\eta$. 
By the simple random sample (i.e., sampling without replacement), $n\hat{F}_n(\eta)$ follows 
a hyper-geometric distribution with parameter 
$N, n, NF(\eta)$. That is, 
\$
\PP( n\hat{F}_n(\eta) = k) = \frac{\binom{NF(\eta)}{k} \binom{N-NF(\eta)}{n-k}}{\binom{N}{n}}, \quad 0\leq k\leq NF(\eta).
\$
Denoting the c.d.f.~of hypergeometric distribution as 
$\Phi_{\rm HG}(k; N, n, NF(\eta))$, we have 
\$
\PP(\hat\eta \leq \eta )  = 1 - \Phi_{\rm HG}\big(\lceil (n+1)(1-q)\rceil-1; N, n, NF(\eta) \big).
\$
Then, for any $t\in [0,1]$, 
\$
\PP\big(  \hat{\rm Cover} \leq t  \big) 
&= \PP\Bigg(  \frac{N}{N-n} F(\hat\eta) - \frac{\lceil (1-q)(n+1)\rceil}{N-n}\leq t \Bigg) \\ 
&= \PP\Bigg(    F(\hat\eta)   \leq \frac{\lceil (1-q)(n+1)\rceil +  (N-n)t}{N} \Bigg).
\$
Since $F(\cdot)$ is monotonely increasing, 
\$
\PP\big(  \hat{\rm Cover} \leq t  \big) 
= \PP\Bigg(    \hat\eta    \leq F^{-1}\bigg(\frac{\lceil (1-q)(n+1)\rceil +  (N-n)t}{N} \bigg) \Bigg),
\$
where $F^{-1}(s)=\inf\{\eta\colon F(\eta)\geq s\}$ for any $s\in [0,1]$. 
Plugging in the previous results on the distribution of $\hat\eta$, 
we have 
\$
\PP\big(  \hat{\rm Cover} \leq t  \big) 
&= 1 - \Phi_{\rm HG}\bigg(\lceil (n+1)(1-q)\rceil-1; N, n, NF\Big({\textstyle F^{-1}\big(\frac{\lceil (1-q)(n+1)\rceil +  (N-n)t}{N} \big)}\Big) \bigg) \\
&= 1 - \Phi_{\rm HG}\bigg(\lceil (n+1)(1-q)\rceil-1; N, n, N  {\textstyle   \frac{\lceil \lceil (1-q)(n+1)\rceil +  (N-n)t\rceil }{N} }   \bigg) \\
&= 1 - \Phi_{\rm HG}\bigg(\lceil (n+1)(1-q)\rceil-1; N, n,     \lceil \lceil (1-q)(n+1)\rceil +  (N-n)t\rceil     \bigg)\\
&= 1 - \Phi_{\rm HG}\bigg(\lceil (n+1)(1-q)\rceil-1; N, n,  \lceil (1-q)(n+1)\rceil +     \lceil (N-n)t\rceil     \bigg)
\$
where the second equality uses the fact that $F(\eta)\in\{0,1/N,\dots,(N-1)/N,1\}$, hence $F(F^{-1}(s))=\lceil Ns\rceil/N$ for any $s\in[0,1]$. By tower property, such an equation also holds for the unconditional distribution, marginalized over all the randomness. This completes the proof of Theorem~\ref{thm:exch}.
\end{proof}

\subsection{Additional visualization of test-time coverage}\label{app:figures}

In this part, we provide more visualization of the distributions of test time coverage $\hat{\textrm{Cover}}$ under various sample size configurations. We note that such results also apply to standard application of split conformal prediction when the non-conformity score function $V$ is independent of calibration and test samples, so that Assumption~\ref{assump:permute} is satisfied.

Figures~\ref{appendix:pdf_n_005} and~\ref{appendix:pdf_n_01} plot the p.d.f.~of $\hat{\textrm{Cover}}$ for $\alpha=0.05$ and $\alpha=0.1$, respectively, when fixing $n$ and varying the test sample size $m$. The $y$-axis is obtained by computing $\PP(t_{k-1}< \hat{\textrm{Cover}}\leq t_k)/(t_k-t_{k-1})$ at $x=(t_{k-1}+t_k)/2$ for a sequence of evenly spaced $\{t_k\}\in [0,1]$. All figures in this paper for p.d.f.s are obtained in the same way. We see that $\hat{\textrm{Cover}}$ concentrates more tightly around the target value $1-\alpha$ as $m$ and $n$ increases. 

Figures~\ref{appendix:pdf_N_005} and~\ref{appendix:pdf_N_01} plot the p.d.f.~of $\hat{\textrm{Cover}}$ for $\alpha=0.05$ and $\alpha=0.1$, respectively, where we fix $N=m+n$ but vary the calibration sample size $n$. 
This mimics the situation where the total number of nodes on the graph is fixed, while we may have flexibility in collecting data as the calibration set. We observe a tradeoff between the calibration accuracy determined by $n$ and the test-sample concentration determined by $n$. The distribution of $\hat{\textrm{Cover}}$ is more concentrated  around $1-\alpha$ when $m$ and $n$ are relatively balanced.

\begin{figure}[htbp]
    \centering
    \begin{subfigure}[t]{0.32\linewidth}
    \centering
    \includegraphics[width=\linewidth]{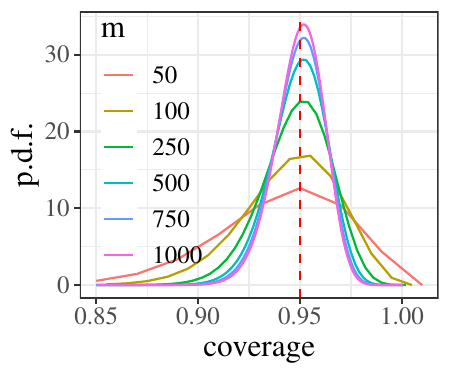}  
\end{subfigure} 
     \begin{subfigure}[t]{0.32\linewidth}
    \centering
    \includegraphics[width=\linewidth]{FIGS/pdf_1000}  
\end{subfigure} 
\begin{subfigure}[t]{0.32\linewidth}
    \centering
    \includegraphics[width=\linewidth]{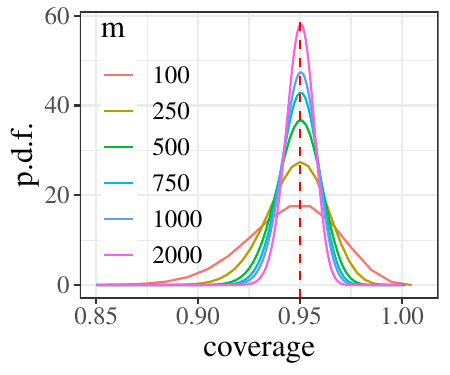}  
\end{subfigure} 
    \caption{P.d.f.~of test-time coverage $\hat{\textrm{Cover}}$ for $n=500$ (left), $1000$ (middle), $2000$ (right) and $\alpha=0.05$ with curves representing different values of $m$, the test sample size.}
    \label{appendix:pdf_n_005}
\end{figure}

\begin{figure}[htbp]
    \centering
    \begin{subfigure}[t]{0.32\linewidth}
    \centering
    \includegraphics[width=\linewidth]{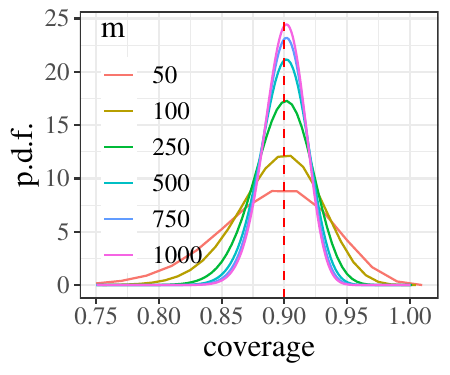}  
\end{subfigure} 
     \begin{subfigure}[t]{0.32\linewidth}
    \centering
    \includegraphics[width=\linewidth]{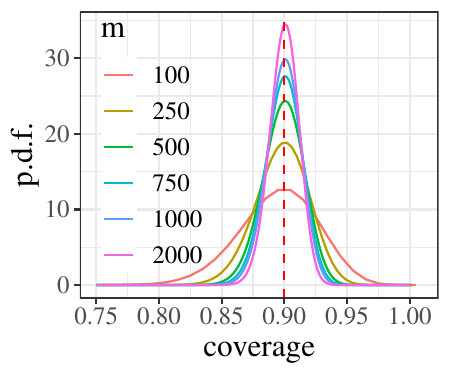}  
\end{subfigure} 
\begin{subfigure}[t]{0.32\linewidth}
    \centering
    \includegraphics[width=\linewidth]{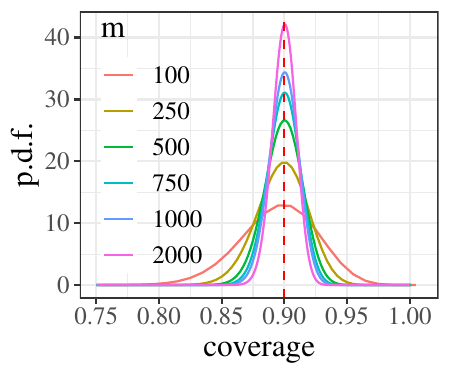}  
\end{subfigure} 
    \caption{P.d.f.~of test-time coverage $\hat{\textrm{Cover}}$ for $n=500$ (left), $1000$ (middle), $2000$ (right) and $\alpha=0.1$ with curves representing different values of $m$, the test sample size.}
    \label{appendix:pdf_n_01}
\end{figure}

\begin{figure}[htbp]
    \centering
    \begin{subfigure}[t]{0.32\linewidth}
    \centering
    \includegraphics[width=\linewidth]{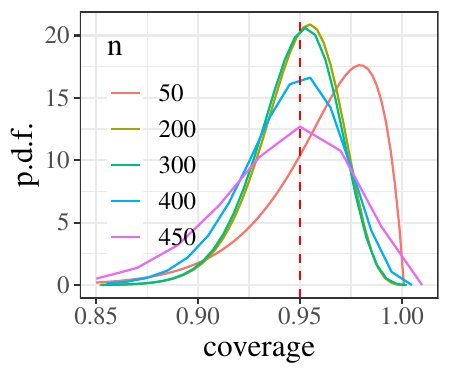}  
\end{subfigure} 
     \begin{subfigure}[t]{0.32\linewidth}
    \centering
    \includegraphics[width=\linewidth]{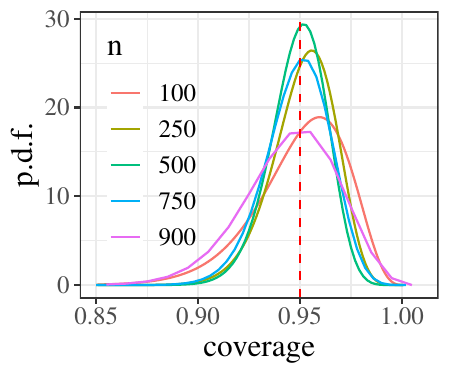}  
\end{subfigure} 
\begin{subfigure}[t]{0.32\linewidth}
    \centering
    \includegraphics[width=\linewidth]{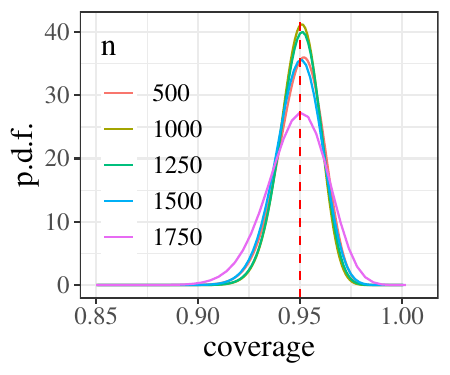}  
\end{subfigure} 
    \caption{P.d.f.~of test-time coverage $\hat{\textrm{Cover}}$ for $N=m+n=500$ (left), $1000$ (middle), $2000$ (right) and $\alpha=0.05$ with curves representing different values of $n$, the calibration sample size.}
    \label{appendix:pdf_N_005}
\end{figure}

\begin{figure}[htbp]
    \centering
    \begin{subfigure}[t]{0.32\linewidth}
    \centering
    \includegraphics[width=\linewidth]{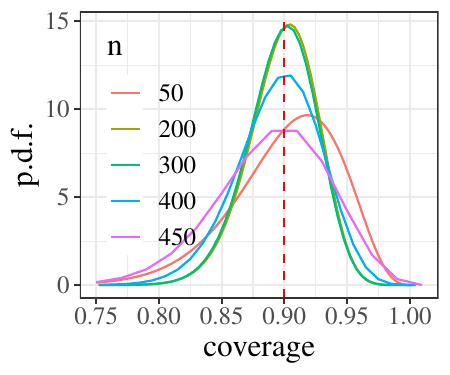}  
\end{subfigure} 
     \begin{subfigure}[t]{0.32\linewidth}
    \centering
    \includegraphics[width=\linewidth]{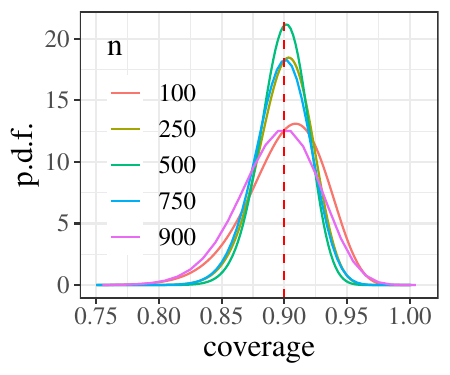}  
\end{subfigure} 
\begin{subfigure}[t]{0.32\linewidth}
    \centering
    \includegraphics[width=\linewidth]{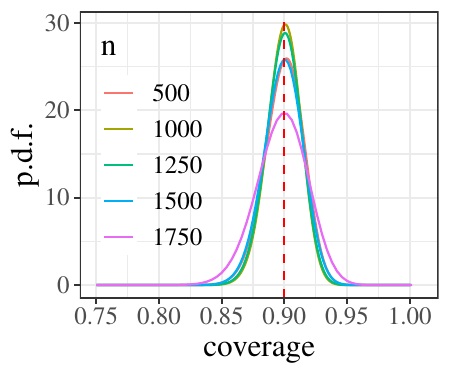}  
\end{subfigure} 
    \caption{P.d.f.~of test-time coverage $\hat{\textrm{Cover}}$ for $N=m+n=500$ (left), $1000$ (middle), $2000$ (right) and $\alpha=0.1$ with curves representing different values of $n$, the calibration sample size.}
    \label{appendix:pdf_N_01}
\end{figure}

\section{Discussion on full conformal prediction, split conformal prediction} \label{appendix:full_vs_split}

In this part, we discuss  the relation of our application of conformal prediction to full conformal prediction and  split conformal prediction, two prominent conformal prediction methods proposed by Vovk and his coauthors in~\cite{vovk2005algorithmic}. Split conformal prediction is mostly widely used due to its computational efficiency, where exchangeability is usually ensured by independence (which is not obvious for graph data) as we discussed briefly in the introduction.

Full conformal prediction (FCP) is arguably the most versatile form of conformal prediction. Given calibration data $Z_i = (X_i,Y_i)\in \cX\times\cY$, $i=1,\dots,n$, and given a test point $X_{n+1}\in\cX$ whose outcome $Y_{n+1}\in \cY$ is unknown, at every hypothesized value $y\in \cY$, FCP uses any algorithm $S$  to train the following scores
\$
S_i^y &= S(X_i,Y_i; Z_1,\dots,Z_{i-1},Z_{i+1},\dots,Z_n, (X_{n+1},y)),\quad i=1,\dots,n,
\$
where $S$ is symmetric in the arguments $Z_1,\dots,Z_{i-1},Z_{i+1},\dots,Z_n, (X_{n+1},y)$, as well as 
\$
S_{n+1}^y &= S(X_{n+1},y; Z_1,\dots,Z_n).
\$

Here, for $1\leq i\leq n$, $S_i^y$ intuitively measures how well the observation $(X_i,Y_i)$ conforms to the observations $Z_1,\dots,Z_{i-1},Z_{i+1},\dots,Z_n, (X_{n+1},y)$ with the hypothesized value of $y$. For instance, when using a linear prediction model, it can be chosen as the prediction residual 
\$
S_i^y =| Y_i - X_i^\top \hat\theta^y|,
\$
where $\hat\theta^y$ is the ordinary least squares coefficient by a linear regression of $Y_1,\dots,Y_{i-1},Y_{i+1},\dots,Y_n,y$ over $X_1,\dots,X_{i-1},X_{i+1},\dots,X_n,X_{n+1}$. 
More generally, one may train a prediction model $\hat\mu^y\colon \cX\to \cY$ using $Z_1,\dots,Z_{i-1},Z_{i+1},\dots,Z_n, (X_{n+1},y)$, and set $S_i^y =| Y_i-\hat\mu^y(X_i)|$. 
For a confidence level $\alpha\in(0,1)$, the FCP prediction set is then 
\$
\hat{C}(X_{n+1}) := \Big\{y \colon \textstyle{\frac{1+ \ind\{S_i^y > S_{n+1}^y\}}{n+1}} \leq \alpha\Big\}.
\$

Since the original form of FCP involves training $n+1$ models at each hypothesized value $y$, its computation can be very intense. It is thus impractical to directly apply FCP to GNN models (i.e., imagining $S$ as the GNN training process on the entire graph with a hypothesized outcome $y$). 

Split conformal prediction (SCP) is a computationally-efficient special case of FCP that is most widely used for i.i.d.~data. The idea is to set aside an independent fold of data to output a single trained model. To be specific, we assume access to a given non-conformity score $V\colon \cX\times\cY\to \RR$, i.i.d.~calibration data $Z_i=(X_i,Y_i)_{i=1}^n$, and an independent test sample $(X_{n+1},Y_{n+1})$ from the same distribution with $Y_{n+1}$ unobserved. 
Here by a ``given'' score, we mean that it is obtained without knowing the calibration and test sample; usually, it is trained on an independent set of data $\{(X_j,Y_j)\}_{j\in \cD_\train}$ before seeing the calibration and test sample. 
Then define $V_i=V(X_i,Y_i)$ for $i=1,\dots,n$. The SCP prediction set is 
\$
\hat{C}(X_{n+1}) = \Big\{y \colon \textstyle{\frac{1+ \ind\{V_i  > V(X_{n+1},y)\}}{n+1}} \leq \alpha\Big\}.
\$
The above set is usually convenient to compute, because we only need one single model to obtain $V$. The validity of SCP usually relies on the independence of $V$ to calibration and test data as we mentioned in the introduction. However, the application of SCP to GNN model is also not straightforward: as we discussed in the main text, the model training step already uses the calibration and test samples, and the nodes are correlated.

Indeed, our method can be seen as a middle ground between FCP and SCP: it only requires one single prediction model as SCP does, but allows to use calibration and test data in the training step as FCP does. 
In our method introduced in the main text, there exists a fixed function $V\colon \cY\times\cY\to \RR$ (provided by APS and CQR) such that 
\$
S_i^y = V(\hat\mu(X_i),Y_i),\quad 
S_{n+1}^y = V(\hat\mu(X_{n+1},y),
\$
where $\hat\mu$ is the final output from the second GNN model whose training process does not utilize the outcomes $Y_1,\dots,Y_n$ and $y$, but uses the features $X_1,\dots,X_n$ and $X_{n+1}$.

\section{Algorithm overview} \label{appendix:algorithm}

We describe the pseudo-code of \mname in Algorithm~\ref{algo1}.

\begin{algorithm}[h]
\SetAlgoLined
\DontPrintSemicolon
\textbf{Input}: Graph $G = (\mathcal{V}, \mathcal{E}, \mathbf{X})$; a trained base GNN model $\mathrm{GNN}_\theta$; non-conformity score function $V(X, Y)$; pre-specified mis-coverage rate $\alpha$, Randomly initialized $\vartheta$ for the conformal correction model $\mathrm{GNN}_\vartheta$.

\While{not done}{
\For{i in \{1, ..., $|\mathcal{V}_\mathrm{cor-calib} \cup \mathcal{V}_\mathrm{cor-test}|$\}}{
$\hat{\mu}(X_i) = \mathrm{GNN}_\theta(\mathbf{X_i}, G)$ \tcp*{Base GNN output scores} 
$\Tilde{\mu}(X_i) = \mathrm{GNN}_\vartheta(\hat{\mu}(X_i), G)$ \tcp*{Correction model output scores}
}
$n, m = |\mathcal{V}_\mathrm{cor-calib}|, |\mathcal{V}_\mathrm{cor-test}|$ \tcp*{Size of correction calib/test set} 
$\hat{\alpha} = \frac{1}{n+1} * \alpha$ \tcp*{Finite-sample correction} 
$\hat{\eta} = \mathrm{DiffQuantile}(\{V(X_i, Y_i) | i \in \mathcal{V}_\mathrm{cor-calib} \})$  \tcp*{Compute non-conformity scores} 
\If{$\mathrm{Classification}$}{
$\mathcal{L}_\mathrm{Ineff} =  \frac{1}{m} \sum_{i \in \mathcal{V}_\mathrm{cor-test}} \frac{1}{|\mathcal{Y}|}\sum_{k \in \mathcal{Y}}\sigma(\frac{V(X_i, k) - \hat{\eta}}{\tau})$ \\ \tcp*{Inefficiency proxy for classification tasks}
}
\If{$\mathrm{Regression}$}{
$\mathcal{L}_\mathrm{Ineff} =  \frac{1}{m} \sum_{i \in \mathcal{V}_\mathrm{cor-test}} (\tilde{\mu}_{1-\alpha/2}(X)_i + \hat{\eta}) - (\tilde{\mu}_{\alpha/2}(X)_i - \hat{\eta})$ \\ \tcp*{Inefficiency proxy for regression tasks} 
$\mathcal{L}_\mathrm{Ineff} += \gamma \frac{1}{m} \sum_{i \in \mathcal{V}_\mathrm{cor-test}}  (\tilde{\mu}_{1-\alpha/2}(X)_i - \hat{\mu}_{1-\alpha/2}(X)_i)^2 + (\tilde{\mu}_{\alpha/2}(X)_i - \hat{\mu}_{\alpha/2}(X)_i)^2$ \\ \tcp*{Consistency regularization term}
}
$\vartheta \leftarrow \vartheta - \nabla_\vartheta \mathcal{L}_\mathrm{Ineff}$ \tcp*{Optimizing $\vartheta$ to reduce inefficiency}
}

\caption{Pseudo-code for \mname algorithm.}
\label{algo1}
\end{algorithm}

\section{Deferred details for experiments}

\subsection{Hyperparameters} \label{appendix:hyperparam}
Table~\ref{tab:hyperparam} reports our set of hyperparameter ranges. We conduct 100 iterations of Bayesian Optimization for \mname with the validation set inefficiency proxy as the optimization metric. To avoid overfitting, each iteration only uses the first GNN run. The optimized hyperparameters are then used for all 10 GNN runs and we then reported the average and standard deviation across runs. Each experiment is done with a single NVIDIA 2080 Ti RTX 11GB GPU.

\begin{table}[h]
    \centering
    \caption{Hyperparameter range for \mname.}
    \begin{tabular}{l|l|l}
    \toprule
    Task & Param. & Range \\ \midrule
    \multirow{5}{*}{Classification}  & $\mathrm{GNN}_\vartheta$ Hidden dimension  & [16,32,64,128,256] \\
    & Learning rate & [1e-1, 1e-2, 1e-3, 1e-4] \\
    & $\mathrm{GNN}_\vartheta$ Number of GNN Layers & [1,2,3,4] \\
    & $\mathrm{GNN}_\vartheta$ Base Model & [GCN, GAT, GraphSAGE, SGC]\\
    & $\tau$ & [10, 1, 1e-1, 1e-2, 1e-3] \\ \midrule
    \multirow{5}{*}{Regression}  & $\mathrm{GNN}_\vartheta$ Hidden dimension  & [16,32,64,128,256] \\
    & Learning rate & [1e-1, 1e-2, 1e-3, 1e-4] \\
    & $\mathrm{GNN}_\vartheta$ Number of GNN Layers & [1, 2, 3, 4] \\
    & $\mathrm{GNN}_\vartheta$ Base Model & [GCN, GAT, GraphSAGE, SGC]\\
    & Reg. loss coeff. $\gamma$ & [1, 1e-1] \\ \bottomrule
    \end{tabular}
    \label{tab:hyperparam}
\end{table}

\subsection{Baseline Details}\label{appendix:baseline}

We report the details about baselines below and the hyperparameter range in Table~\ref{tab:hyperparam_baseline}.

\begin{enumerate}
    \item Temperature Scaling~\cite{guo2017calibration} divides the logits with a learnable scalar. It is optimized over NLL loss in the validation set. 
    \item Vector Scaling~\cite{guo2017calibration} has a scalar to scale the logits for each class dimension and adds an additional classwide bias. It is optimized over NLL loss in the validation set. 
    \item Ensemble Temperature Scaling~\cite{zhang2020mix} learns an ensemble of uncalibrated, temperature-scaled calibrated calibrators. 
    \item CaGCN~\cite{wang2021confident} uses an additional GCN model that learns a temperature scalar for each node based on its neighborhood information. 
    \item GATS~\cite{hsu2022makes} identifies five factors that affect GNN calibration and designs a model that accounts for these five factors by using per-node temperature scaling and attentive aggregation from the local neighborhood.
    \item QR~\cite{koenker2001quantile} uses a pinball loss to produce quantile scores. It is CQR without the conformal prediction adjustment. 
    \item MC dropout~\cite{gal2016dropout} turns on dropout during evaluation and produces $K$ predictions. We then take the 95\% quantile of the predicted distribution. We also experimented with taking a 95\% confidence interval but 95\% quantile has better coverage, thus we adopt the quantile approach.
    \item BayesianNN~\cite{kendall2017uncertainties} model the label with normal distribution and the model produces two heads, where one corresponds to the mean and the second log variance. We then calculate the standard deviation as the square root of the exponent of log variance. Then we take the [mean-1.96*standard deviation, mean+1.96*standard deviation] for the 95\% interval. 
\end{enumerate}

\begin{table}[h]
    \centering
    \caption{Hyperparameter range for baselines.}
    \begin{tabular}{l|l|l}
    \toprule
    Baseline & Param. & Range \\ \midrule
    Temperature Scaling  & No hyperparameter  & Not Applicable \\ \midrule
    Vector Scaling & No hyperparameter  & Not Applicable \\ \midrule
    Ensemble Temp Scaling & No hyperparameter  & Not Applicable \\ \midrule
    \multirow{4}{*}{CaGCN} & Dropout & [0.3, 0.5, 0.7] \\
    & Hidden dimension & [16, 32, 64, 128, 256] \\
    & Number of GNN Layers & [1,2,3,4] \\ 
    & Weight Decay & [0, 1e-3, 1e-2, 1e-1] \\  \midrule
    \multirow{4}{*}{GATS} & Dropout & [0.3, 0.5, 0.7] \\
    & Hidden dimension & [16, 32, 64, 128, 256] \\
    & Number of GNN Layers & [1,2,3,4] \\
    & Weight Decay & [0, 1e-3, 1e-2, 1e-1] \\  \midrule
    MC Dropout & Number of Predictions & [100, 500, 1,000] \\ \midrule
    BayesianNN & No hyperparameter  & Not Applicable \\
    \bottomrule
    
    \end{tabular}
    \label{tab:hyperparam_baseline}
\end{table}

\subsection{Dataset} \label{appendix:dataset}

For node classification, we use the common node classification datasets in Pytorch Geometric package. For node regression, we use datasets in~\cite{jia2020residual}. We report the dataset statistics at Table~\ref{appendix:tab_data}.  

\begin{table}[h]
    \centering
    \caption{Dataset statistics.}
    \adjustbox{max width=0.95\textwidth}{
    \begin{tabular}{l|l|l|llll}
    \toprule
    Domain & Dataset & Task & \# Nodes & \# Edges & \# Features & \# Labels \\ \midrule
    \multirow{4}{*}{Citation} & Cora & Classification & 2,995 & 16,346 & 2,879 & 7 \\
    & DBLP & Classification & 17,716 & 105,734 & 1,639 & 4 \\
    & CiteSeer & Classification & 4,230 & 10,674 & 602 & 6 \\
    & PubMed & Classification & 19,717 & 88,648 & 500 & 3 \\\cline{1-1}
    \multirow{2}{*}{Co-purchase}&Computers & Classification & 13,752 & 491,722 & 767 & 10 \\
    &Photos & Classification & 7,650 & 238,162 & 745 & 8 \\\cline{1-1}
    \multirow{2}{*}{Co-author}&CS & Classification & 18,333 & 163,788 & 6,805 & 15 \\
    &Physics & Classification & 34,493 & 495,924 & 8,415 & 5 \\\midrule%
    \multirow{2}{*}{Transportation}&Anaheim & Regression & 914 & 3,881 & 4 & -- \\
    &Chicago & Regression & 2,176 & 15,104 & 4 & -- \\\cline{1-1}
    \multirow{4}{*}{Geography}&Education & Regression & 3,234 & 12,717 & 6 & -- \\
    &Election & Regression & 3,234 & 12,717 & 6 & -- \\
    &Income & Regression & 3,234 & 12,717 & 6 & -- \\
    &Unemployment & Regression & 3,234 & 12,717 & 6 & -- \\\cline{1-1}
    Social&Twitch & Regression & 1,912 & 31,299 & 3,170 & -- \\ \bottomrule
    \end{tabular}
    }
    \label{appendix:tab_data}
\end{table}

\begin{figure}[h]
    \centering
    \includegraphics[width = \textwidth]{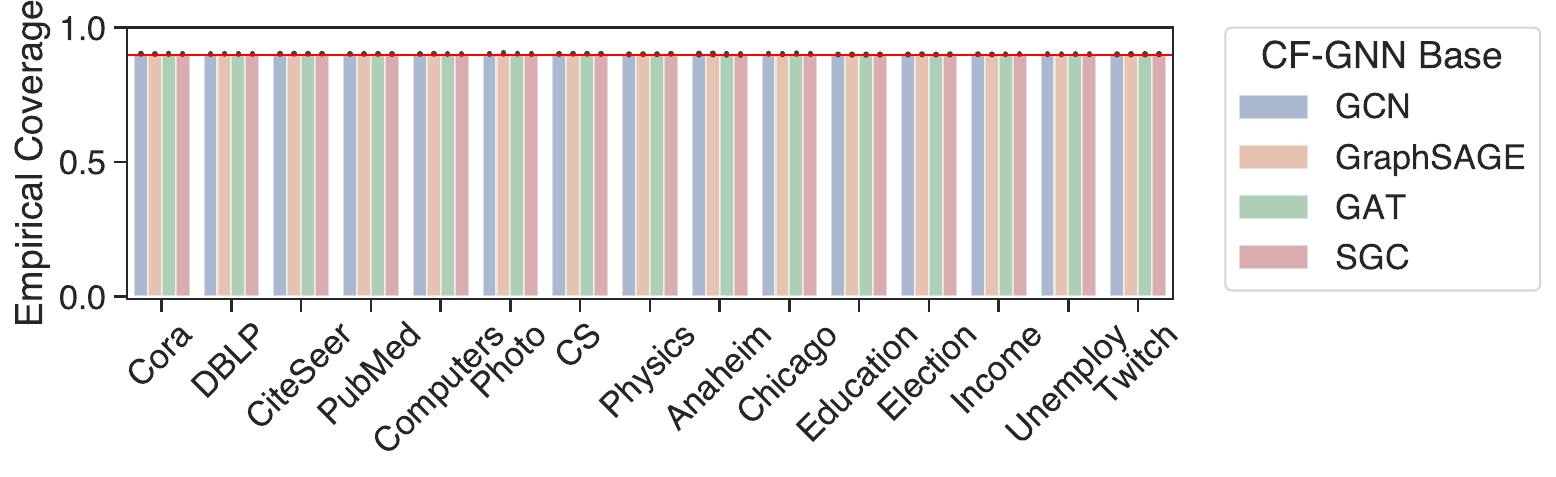}
    \caption{Empirical coverage across 15 datasets with 10 independent runs of GNN, using \mname. }
    \label{fig:coverage_empirical}
\end{figure}

\subsection{Marginal coverage and inefficiency across GNN architectures} \label{appendix:gnn_architectures}

We additionally conduct marginal coverage and inefficiency comparisons of \mname over the vanilla CP across 4 different GNN architectures: GCN, GAT, GraphSAGE, and SGC. The result for marginal coverage is in Figure~\ref{fig:coverage_empirical}. The result for inefficiency is in Table~\ref{tab:eff_all_models}. We observe consistent improvement in inefficiency reduction across these architectures, suggesting \mname is a GNN-agnostic efficiency improvement approach.

\begin{table}[t]
    \centering
    \caption{Empirical inefficiency measured by the size/length of the prediction set/interval for node classification/regression. The result uses APS for classification and CQR for regression. We report the average and standard deviation calculated from 10 GNN runs with each run of 100 conformal splits.}
    \begin{adjustbox}{width=\textwidth}
\begin{tabular}{l|l|c|c|c|c}
\toprule
\multirow{2}{*}{Task} & GNN Model & GCN & GraphSAGE & GAT & SGC \\ \cmidrule{2-6}
& Dataset & CP$\xrightarrow{\hspace*{0.7cm}}$CF-GNN & CP$\xrightarrow{\hspace*{0.7cm}}$CF-GNN & CP$\xrightarrow{\hspace*{0.7cm}}$CF-GNN & CP$\xrightarrow{\hspace*{0.7cm}}$CF-GNN \\\midrule
\multirow{8}{*}{\parbox{1.2cm}{Node classif.}} & Cora&3.80\std{.28}$\xrightarrow{-58.44\%}$1.58\std{.22}&6.73\std{.19}$\xrightarrow{-76.50\%}$1.58\std{.15}&4.14\std{.16}$\xrightarrow{-62.53\%}$1.55\std{.10}&3.88\std{.19}$\xrightarrow{-62.23\%}$1.47\std{.10}\\
&DBLP&2.43\std{.03}$\xrightarrow{-49.20\%}$1.23\std{.01}&3.91\std{.01}$\xrightarrow{-68.27\%}$1.24\std{.01}&2.02\std{.06}$\xrightarrow{-37.51\%}$1.26\std{.01}&2.44\std{.04}$\xrightarrow{-48.93\%}$1.24\std{.02}\\
&CiteSeer&3.86\std{.11}$\xrightarrow{-74.50\%}$0.99\std{.01}&5.88\std{.02}$\xrightarrow{-83.07\%}$1.00\std{.01}&3.18\std{.25}$\xrightarrow{-68.56\%}$1.00\std{.01}&3.79\std{.14}$\xrightarrow{-73.43\%}$1.01\std{.02}\\
&PubMed&1.60\std{.02}$\xrightarrow{-19.44\%}$1.29\std{.04}&1.93\std{.28}$\xrightarrow{-36.95\%}$1.22\std{.03}&1.37\std{.02}$\xrightarrow{-10.54\%}$1.23\std{.02}&1.60\std{.02}$\xrightarrow{-24.27\%}$1.21\std{.02}\\
&Computers&3.56\std{.13}$\xrightarrow{-50.22\%}$1.77\std{.11}&6.00\std{.10}$\xrightarrow{-45.74\%}$3.26\std{.48}&2.33\std{.11}$\xrightarrow{-8.18\%}$2.14\std{.12}&3.44\std{.14}$\xrightarrow{-45.69\%}$1.87\std{.15}\\
&Photo&3.79\std{.13}$\xrightarrow{-57.03\%}$1.63\std{.17}&4.52\std{.47}$\xrightarrow{-37.22\%}$2.84\std{.67}&2.24\std{.21}$\xrightarrow{-19.38\%}$1.81\std{.16}&3.81\std{.14}$\xrightarrow{-62.86\%}$1.41\std{.05}\\
&CS&7.79\std{.29}$\xrightarrow{-55.83\%}$3.44\std{.33}&14.68\std{.02}$\xrightarrow{-88.63\%}$1.67\std{.14}&6.87\std{.48}$\xrightarrow{-73.08\%}$1.85\std{.10}&7.76\std{.25}$\xrightarrow{-73.92\%}$2.02\std{.22}\\
&Physics&3.11\std{.07}$\xrightarrow{-65.36\%}$1.08\std{.10}&4.91\std{.01}$\xrightarrow{-72.97\%}$1.33\std{.08}&2.00\std{.19}$\xrightarrow{-45.23\%}$1.09\std{.06}&3.10\std{.08}$\xrightarrow{-57.22\%}$1.32\std{.12}\\ \midrule
\multicolumn{2}{l|}{Average Improvement} & -53.75\% & -63.75\% & -40.63\% & -56.07\% \\\midrule
\multirow{8}{*}{\parbox{1cm}{Node regress.}} &Anaheim&2.89\std{.39}$\xrightarrow{-25.00\%}$2.17\std{.11}&2.37\std{.05}$\xrightarrow{-23.12\%}$1.82\std{.07}&3.12\std{.38}$\xrightarrow{-31.27\%}$2.14\std{.11}&2.94\std{.24}$\xrightarrow{-24.90\%}$2.21\std{.16}\\
&Chicago&2.05\std{.07}$\xrightarrow{-0.48\%}$2.04\std{.17}&2.08\std{.05}$\xrightarrow{-7.90\%}$1.92\std{.09}&1.95\std{.04}$\xrightarrow{-68.15\%}$0.62\std{.93}&2.02\std{.03}$\xrightarrow{-1.37\%}$1.99\std{.07}\\
&Education&2.56\std{.02}$\xrightarrow{-5.07\%}$2.43\std{.05}&2.20\std{.04}$\xrightarrow{+8.44\%}$2.38\std{.08}&2.48\std{.05}$\xrightarrow{-2.76\%}$2.41\std{.04}&2.55\std{.02}$\xrightarrow{-2.80\%}$2.48\std{.04}\\
&Election&0.90\std{.01}$\xrightarrow{+0.21\%}$0.90\std{.02}&0.87\std{.01}$\xrightarrow{-0.80\%}$0.86\std{.02}&0.89\std{.00}$\xrightarrow{-1.23\%}$0.88\std{.02}&0.90\std{.00}$\xrightarrow{-0.42\%}$0.90\std{.02}\\
&Income&2.51\std{.12}$\xrightarrow{-4.58\%}$2.40\std{.05}&2.08\std{.04}$\xrightarrow{+32.23\%}$2.75\std{.23}&2.35\std{.03}$\xrightarrow{-0.23\%}$2.34\std{.07}&2.42\std{.01}$\xrightarrow{+3.04\%}$2.49\std{.04}\\
&Unemploy.&2.72\std{.03}$\xrightarrow{-10.83\%}$2.43\std{.04}&2.75\std{.06}$\xrightarrow{-12.90\%}$2.39\std{.05}&2.80\std{.08}$\xrightarrow{-14.56\%}$2.40\std{.04}&2.72\std{.02}$\xrightarrow{-11.05\%}$2.42\std{.04}\\
&Twitch&2.43\std{.10}$\xrightarrow{-1.36\%}$2.39\std{.07}&2.48\std{.09}$\xrightarrow{-3.06\%}$2.40\std{.07}&2.50\std{.14}$\xrightarrow{-5.53\%}$2.36\std{.07}&2.42\std{.08}$\xrightarrow{-1.43\%}$2.38\std{.06}\\ \midrule
\multicolumn{2}{l|}{Average Improvement} &-6.73\% & -1.02\% & -17.68\% & -5.56\% \\
 \bottomrule
\end{tabular}    
    \label{tab:eff_all_models}
    \end{adjustbox}
\end{table}

\subsection{\mname with Regularized Adaptive Prediction Sets} \label{appendix:raps}

To further showcase that \mname is a versatile framework that adapts to any advancement in non-conformity scores, we experiment on RAPS~\cite{angelopoulos2020uncertainty}, which regularizes APS to produce a smaller prediction set size. We report the performance using the GCN backbond in Table~\ref{appendix:tab_raps}. We observe that \mname still obtains impressive inefficiency reduction compared to the vanilla application of RAPS to GNN.

\begin{table}[h]
    \centering
    \caption{Comparison with other non-conformity scores that reduce inefficiency. }
    \begin{tabular}{l|c}
    \toprule
    Size  & CP $\longrightarrow$ CF-GNN \\ \midrule
    Cora&1.67\std{.11}$\xrightarrow{-15.35\%}$1.42\std{.05}\\
    DBLP&1.39\std{.02}$\xrightarrow{-5.00\%}$1.32\std{.01}\\
    CiteSeer&1.30\std{.07}$\xrightarrow{-19.85\%}$1.04\std{.04}\\
    PubMed&1.23\std{.01}$\xrightarrow{+2.40\%}$1.26\std{.02}\\
    Computers&1.58\std{.02}$\xrightarrow{-4.59\%}$1.51\std{.05}\\
    Photo&1.34\std{.01}$\xrightarrow{-10.47\%}$1.20\std{.01}\\
    CS&1.29\std{.04}$\xrightarrow{-6.13\%}$1.21\std{.02}\\
    \bottomrule
    \end{tabular}
    \label{appendix:tab_raps}
\end{table}

\subsection{Conditional coverage on full set of network features} \label{appendix:cond_cov_networks}
We report the full set of network features and calculate the worse-slice coverage in Table~\ref{tab:full_network_cond_cov} for target coverage of 0.9 and Table~\ref{tab:full_network_cond_cov_95} for a target coverage of 0.95. We observe that \mname achieves satisfactory conditional coverage across a wide range of diverse network features.

\begin{table}[h]
\centering
    \captionof{table}{\mname achieves conditional coverage. We use Cora/Twitch as an example classification/regression dataset. }
    \adjustbox{max width=0.95\textwidth}{
    \begin{tabular}{l|cc|cc}
    \toprule
    Target: 0.9 & \multicolumn{2}{c|}{Classification} & \multicolumn{2}{c}{Regression} \\ \midrule
    Model & CP & CF-GNN & CP & CF-GNN  \\ \midrule
    Marginal Cov. & 0.90\std{.02} & 0.90\std{.01} & 0.91\std{.02} & 0.91\std{.03}\\ \midrule
    Cond. Cov. (Input Feat.) & 0.89\std{.04} & 0.90\std{.03} & 0.90\std{.07} & 0.86\std{.08} \\ \midrule
    Cond. Cov. (Cluster) & 0.82\std{.07} & 0.89\std{.03} & 0.90\std{.06} & 0.88\std{.07} \\
    Cond. Cov. (Between) & 0.82\std{.06} & 0.89\std{.03} &0.86\std{.08}&0.88\std{.07}\\
    Cond. Cov. (PageRank) & 0.71\std{.08} & 0.87\std{.05} & 0.87\std{.09} & 0.89\std{.07}\\
    Cond. Cov. (Load) &0.83\std{.05}&0.90\std{.03}&0.86\std{.08}&0.88\std{.07}\\
    Cond. Cov. (Harmonic) &0.89\std{.04}&0.87\std{.05}&0.88\std{.08}&0.91\std{.06}\\
    Cond. Cov. (Degree) & 0.79\std{.05} & 0.89\std{.04} & 0.86\std{.08} & 0.89\std{.06} \\    
    \bottomrule
    \end{tabular}
    \label{tab:full_network_cond_cov}
    }

\end{table}

\begin{table}[h]
\centering
    \captionof{table}{\mname achieves conditional coverage. We use Cora/Twitch as an example classification/regression dataset. }
    \adjustbox{max width=0.95\textwidth}{
    \begin{tabular}{l|cc|cc}
    \toprule
    Target: 0.95 & \multicolumn{2}{c|}{Classification} & \multicolumn{2}{c}{Regression} \\ \midrule
    Model & CP & CF-GNN & CP & CF-GNN  \\ \midrule
    Marginal Cov. & 0.95\std{.01} & 0.95\std{.01} & 0.96\std{.02} & 0.96\std{.02}\\ \midrule
    Cond. Cov. (Input Feat.) & 0.94\std{.02} & 0.94\std{.03} & 0.95\std{.04} & 0.94\std{.05} \\ \midrule
    Cond. Cov. (Cluster) & 0.89\std{.06} & 0.93\std{.04} & 0.96\std{.03} & 0.96\std{.03} \\
    Cond. Cov. (Between) & 0.81\std{.06} & 0.95\std{.03} & 0.94\std{.05} & 0.94\std{.05} \\
    Cond. Cov. (PageRank) & 0.78\std{.06} & 0.94\std{.03} & 0.94\std{.05} & 0.94\std{.05}\\
    Cond. Cov. (Load) &0.81\std{.06}&0.94\std{.03}&0.94\std{.05}&0.95\std{.05}\\
    Cond. Cov. (Harmonic) &0.88\std{.04}&0.95\std{.03}&0.96\std{.04}&0.95\std{.04}\\
    Cond. Cov. (Degree) & 0.83\std{.05} &0.88\std{.06} & 0.94\std{.04} & 0.94\std{.04} \\    
    \bottomrule
    \end{tabular}
    \label{tab:full_network_cond_cov_95}
    }

\end{table}

\subsection{Prediction accuracy versus uncertainty calibration} \label{appendix:accuracy_prediction}

As we discussed in the main text, the original GNN trained towards optimal prediction accuracy does not necessarily yield the most efficient prediction model; this is corrected by the second GNN in CF-GNN which improves the efficiency of conformal prediction sets/intervals. 
With our approach, one can use the output of the original GNN for point prediction while that of the second GNN for efficient uncertainty quantification, without necessarily overwriting the first accurate prediction model. 
However, a natural question here still remains, which is that whether applying the second GNN drastically changes the prediction accuracy. This question is more relevant to the classification problem since for regression our method only adjusts the confidence band. For classification, we consider  top-1 class prediction as the ``point prediction''. We present its accuracy "Before" and "After" the correction in Table~\ref{tab:accuracy_change}, which shows that this correction typically does not result in a visible change in accuracy. In addition, in a new experiment on Cora, we find that 100\% of the top-1 class from the base GNN are in CF-GNN’s prediction sets. The potential to develop steps that explicitly consider point prediction accuracy is an exciting avenue for future research.

\begin{table}[h]
\centering
    \captionof{table}{\mname does not change the top-1 class prediction accuracy for classification tasks. }
    \adjustbox{max width=0.95\textwidth}{
    \begin{tabular}{l|c|c}
    \toprule
    Dataset & Before & After \\  \midrule
    Cora	& 0.844\std{0.004} &	 0.843\std{0.016} \\ 
DBLP	& 0.835\std{0.001}	& 0.832\std{0.002} \\
CiteSeer&	 0.913\std{0.002}	& 0.911\std{0.002}
     \\    
    \bottomrule
    \end{tabular}
    \label{tab:accuracy_change}
    }

\end{table}

\section{Extended Related Works}\label{appendix:related_work}

\xhdr{Uncertainty quantification for graph neural networks} Uncertainty quantification (UQ) is a well-studied subject in general machine learning and also recently in GNNs. For multi-class classification, the raw prediction scores are often under/over-confident and thus various calibration methods are proposed for valid uncertainty estimation such as temperate scaling~\cite{guo2017calibration}, vector scaling~\cite{guo2017calibration}, ensemble temperate scaling~\cite{zhang2020mix}, and so on~\cite{gupta2020calibration,kull2019beyond,stadler2021graph,abdar2021review}. Recently, specialized calibration methods that leverage network principles such as homophily \revise{have} been developed: examples include CaGCN~\cite{wang2021confident} and GATS~\cite{hsu2022makes}. In regression, \revise{various methods have been proposed to construct} prediction intervals, such as quantile regression~\cite{koenker2001quantile,takeuchi2006nonparametric,sheather1990kernel}, bootstrapping with subsampling,  model ensembles, and dropout initialization~\cite{gal2016dropout,lakshminarayanan2017simple,kuleshov2018accurate,ovadia2019can}, and bayesian approaches with strong modeling assumptions on parameter and data distributions~\cite{kendall2017uncertainties,izmailov2021bayesian}. However, these UQ methods can fail to provide statistically rigorous and empirically valid coverage guarantee (see Table~\ref{tab:cov}). In contrast, \mname achieves valid marginal coverage  in both theory and practice. Uncertainty quantification has also been leveraged to deal with out-of-distribution detection and imbalanced data in graph neural networks~\cite{zhao2020uncertainty,gao2023topology}. While it is not the focus here, we remark that conformal prediction can also be extended to tackle such issues~\cite{ishimtsev2017conformal}, and it would be interesting to explore such applications for graph data.

\xhdr{Conformal prediction for graph neural networks} {As we discussed, the application of conformal prediction to graph-structured data remains largely unexplored. At the time of submission, the only work we awared of is \cite{clarkson2022distribution}, who claims that nodes in the graph are not exchangeable in the inductive setting  and employs the framework of~\cite{barber2022conformal} to construct conformal prediction sets using neighborhood nodes as the calibration data. In contrast, we study the transductive setting  where certain exchangeablility property holds and allows for flexibility in the training step. We also study the efficiency aspect that is absent in~\cite{clarkson2022distribution}. In addition, there have been concurrent works~\cite{pmlr-v202-h-zargarbashi23a,lunde2023conformal} that observe similar exchangeability and validity of conformal prediction in either transductive setting or other network models. In particular,~\cite{pmlr-v202-h-zargarbashi23a} proposes a diffusion-based method that aggregates non-conformity scores of neighbor nodes to improve efficiency, while our approach learns the aggregation of neighbor scores, which is more general than their approach. \cite{lunde2023validity} studies the exchangeability for node regression under certain network models instead of our transductive setting with GNNs, and without considering the efficiency aspect. With a growing recent interest in conformal prediction for graphs, there are even more recent works that focus on validity~\cite{lunde2023validity} and  link prediction~\cite{marandon2023conformal}.

\xhdr{Efficiency of conformal prediction} \revise{While conformal prediction enjoys distribution-free coverage for any non-conformity score based on any prediction model, its efficiency (i.e., size of prediction sets or length of prediction intervals) varies with specific choice of the scores and models. How to achieve desirable properties such as efficiency is a topic under intense research in conformal prediction. To this end, one major thread designs good non-conformity scores such as APS~\cite{romano2020classification} and CQR~\cite{romano2019conformalized}. More recent works take another approach, by modifying the training process of the prediction model to further improve  efficiency.
This work falls into the latter case. Our idea applies to any non-conformity scores, as demonstrated with APS and CQR, two prominent examples of the former case. Related to our work, ConfTr~\cite{stutz2021learning} also simulates conformal prediction so as to train a prediction model that eventually leads to more efficient conformal prediction sets.
However, our approach differs from theirs in significant ways. First, ConfTr modifies model training, while \mname conducts post-hoc correction without changing the original prediction. Second, ConfTr uses the training set to simultaneously optimize model prediction and efficiency of conformal prediction, while we withhold a fraction of calibration data to optimize the efficiency. Third, our approach specifically leverages the rich topological information in graph-structured data to achieve more improvement in efficiency. Finally, we also propose a novel loss for efficiency in regression tasks.}

\end{document}